\documentclass[11pt,oneside]{article}

\usepackage[numbers]{natbib}

\usepackage[utf8]{inputenc}
\usepackage[T1]{fontenc}
\usepackage{nicefrac}
\usepackage{amsmath}
\usepackage{amssymb,amsfonts}
\usepackage{bm}

\usepackage[paper=a4paper, hscale=0.75, vscale=0.79]{geometry}

\usepackage[dvipsnames]{xcolor}

\usepackage{graphicx}
\graphicspath{{./}}
\DeclareGraphicsExtensions{.pdf,.png,.jpg,.eps}
\usepackage{subfig}

\newcommand{\mv}[1]{\bm{#1}}
\newcommand*\dif{\mathop{}\!\mathrm{d}} 

\DeclareMathOperator{\exx}{\mathbf{E}}

\DeclareMathOperator{\prr}{\mathbf{P}}

\DeclareMathOperator{\vaa}{var}
\DeclareMathOperator*{\argmin}{arg\,min}

\def\CC{\mathcal{C}}

\def\OO{\mathcal{O}}

\def\RR{\mathbb{R}}

\def\WW{\mathcal{W}}

\def\ZZ{\mathcal{Z}}

\def\Z{\mv{Z}}

\newcommand{\overbar}[1]{\mkern 1.5mu\overline{\mkern-1.5mu#1\mkern-1.5mu}\mkern 1.5mu}

\newcommand*{\defeq}{\mathrel{\vcenter{\baselineskip0.5ex \lineskiplimit0pt     
                     \hbox{\scriptsize.}\hbox{\scriptsize.}}}
                     =}

\newcommand{\term}[1]{\textcolor{BlueViolet}{\textit{{#1}}}}

\DeclareMathOperator{\bregbase}{D} 
\def\breg{\bregbase_{\Phi}} 
\def\cond{\,|\,} 
\DeclareMathOperator{\crisk}{CVaR} 
\def\ddist{\textup{P}} 
\DeclareMathOperator{\dfun}{F} 
\DeclareMathOperator{\dfunhat}{\widehat{F}} 
\DeclareMathOperator{\ddfunhat}{\widehat{f}} 
\def\diameter{\Delta} 
\DeclareMathOperator{\idc}{I} 
\DeclareMathOperator{\loss}{L} 
\def\ndist{\nu} 
\def\pert{\gamma} 
\DeclareMathOperator{\risk}{R} 
\def\smooth{\lambda} 
\DeclareMathOperator{\sloss}{r_{\spec}} 
\DeclareMathOperator{\slhat}{\widehat{r}_{\spec}} 
\DeclareMathOperator{\spec}{\sigma} 
\DeclareMathOperator{\srisk}{R_{\spec}} 
\DeclareMathOperator{\ssrisk}{\widetilde{R}_{\spec}} 
\def\strong{\kappa} 
\DeclareMathOperator{\vrisk}{VaR} 

\usepackage{amsthm}
\theoremstyle{definition} \newtheorem{defn}{Definition}
\theoremstyle{plain} 
\theoremstyle{plain} \newtheorem{thm}[defn]{Theorem}
\theoremstyle{plain} \newtheorem{lem}[defn]{Lemma}
\theoremstyle{plain} 
\theoremstyle{remark} \newtheorem{rmk}[defn]{Remark}
\theoremstyle{remark}

\usepackage[pdfauthor={Matthew J. Holland},%
colorlinks=true,%
linkcolor=blue,%
citecolor=blue]{hyperref}
\hypersetup{pdftitle={Spectral risk-based learning using unbounded losses}}

\usepackage{url}
\usepackage{booktabs}
\usepackage{amsfonts}
\usepackage{microtype}
\usepackage{algorithm}
\usepackage{algpseudocode}
\usepackage{enumitem}
\makeatletter
\def\namedlabel#1#2{\begingroup
    #2%
    \def\@currentlabel{#2}%
    \phantomsection\label{#1}\endgroup
}
\makeatother

\title{\textbf{Spectral risk-based learning using unbounded losses}}
\author{
  Matthew J.~Holland\thanks{Please direct correspondence to \texttt{matthew-h@ar.sanken.osaka-u.ac.jp}.}\\
  Osaka University
  \and
  El Mehdi Haress\\
  Universit\'{e} Paris-Saclay
}
\date{} 

\begin{document}

\maketitle

\begin{abstract}
In this work, we consider the setting of learning problems under a wide class of spectral risk (or ``L-risk'') functions, where a Lipschitz-continuous spectral density is used to flexibly assign weight to extreme loss values. We obtain excess risk guarantees for a derivative-free learning procedure under unbounded heavy-tailed loss distributions, and propose a computationally efficient implementation which empirically outperforms traditional risk minimizers in terms of balancing spectral risk and misclassification error.
\end{abstract}

\tableofcontents

\section{Introduction}\label{sec:intro}

While the choice of loss function is a fundamental part of the daily workflow of most users of machine learning systems, the question of which \emph{risk} to use receives far less attention. This is likely due to a tacit acceptance that the abstract notion of ``good generalization ability'' is best formulated by the expected loss $\exx_{\ddist}\loss(w;Z)$, where $Z \sim \ddist$ is a random observation, and $w$ characterizes some decision rule. While the influential learning models of \citet{vapnik1999NSLT} and \citet{haussler1992a} are centered around the expected loss, it can be argued that prioritizing \emph{average} off-sample performance is a substantial value judgement that requires more serious consideration, both by stakeholders involved in the practical side of machine learning systems, and by the theoretician interested in providing learning algorithms with formal guarantees, stated in terms of whatever ``risk'' is chosen.

In the last few years, notable progress has been made in terms of learning under non-traditional risks. By far the most well-studied variant is the \term{conditional value-at-risk} (CVaR) of the loss distribution. Numerous applications to CVaR-based sequential learning problems have been studied \citep{galichet2013a,tamar2015a,prashanth2020a}. More recently, under convex losses, finite-sample excess (CVaR) risk bounds for stochastic gradient-based learning algorithms have been obtained under essentially any loss distribution \citep{holland2021c,soma2020a}, while adaptive sampling strategies have been used to improve robustness to distributional shift, without relying on convexity \citep{curi2020a}. Unfortunately, despite its practical utility, CVaR is very restrictive in terms of expressible risk preferences; all losses beyond a pre-fixed threshold are given the exact same weight. A well-known generalization is the class of \term{spectral risks} \citep{acerbi2002a}, which utilize a non-constant weighting function. This dramatically improves flexibility, but comes at the cost of more complicated form, which is expensive to estimate and difficult to optimize using traditional first-order stochastic descent methods.

To address this issue, we start by taking a derivative-free approach to learning with spectral risks. Using a stochastic smoothing technique, we first derive finite-sample excess spectral risk bounds in expectation for the proposed procedure (section \ref{sec:theory_expectation}), and show how confidence boosting can be used to obtain high-probability guarantees under loss distributions assuming just finite variance (section \ref{sec:theory_highprob}). In section \ref{sec:empirical}, we propose a simple modification to the derivative-free procedure that lets us integrate gradient information from the losses for faster convergence, and we empirically verify that this procedure efficiently achieves a small spectral risk, with the interesting side-effect of out-performing empirical risk minimizers in terms of misclassification error as well, uniformly across several benchmark datasets.

\section{Related work}\label{sec:relatedwork}

\paragraph{Risk and learning}

While the expected value of the loss distribution is central to statistical learning theory \citep{haussler1992a,vapnik1999NSLT}, more diverse notions of risk have been studied in broader contexts, in particular notions of financial risk \citep{artzner1999a,rockafellar2000a,acerbi2002a,ruszczynski2006a} and risks which capture human psychological tendencies of aversion or affection to risk in uncertain decision-making situations \citep{tversky1992a}. In the classical theory of portfolio optimization, the mean-variance notion of risk plays a central role \citep{markowitz1952a}, and variance-regularized stochastic learning algorithms have been studied by \citet{duchi2019a}. As described earlier, originally borrowed from the financial literature, CVaR has seen many direct applications to learning problems, offers a natural interpretation to the $\nu$-SVM algorithm \citep{takeda2008a}, and appears less explicitly in algorithms designed to minimize worst-case losses \citep{shalev2016a,fan2017a}. More recently, classes of generalized location-deviation risks have been studied \citep{lee2020a,holland2021mrisk}, though this goes beyond the traditional setting of \term{coherent risks} \citep{artzner1999a}. In contrast, our study of learning under spectral risks here lets us go well beyond CVaR while still retaining the properties of coherent risks.

\paragraph{Spectral risks in machine learning}

The research on learning with spectral risks is still very limited. Recent work from \citet{bhat2019a} and \citet{pandey2019a} provides estimators for the spectral risk under sub-Gaussian and sub-Exponential loss distributions, but these results are ``pointwise'' in that they can only be applied to pre-fixed candidates (e.g., predictors, clusters, etc.), and do not extend to learning algorithms which consider many candidates in a data-driven fashion. Work from \citet{khim2020a} includes uniform convergence for empirical spectral risks (under the name ``L-risks''), though their analysis is restricted to bounded losses, and does not lead to excess spectral risk guarantees for any particular class of learning algorithms. Our approach in this work does not build directly upon these results, since instead of a traditional empirical risk minimization (ERM) approach, we take the alternative route of optimizing a smoothed variant, whose distance from the desired risk can be readily controlled. This lets us obtain excess risk bounds for an explicit procedure (section \ref{sec:theory_expectation}, Algorithm \ref{algo:smd}), with much weaker assumptions on the underlying loss distribution.

\paragraph{Robustness to heavy-tailed losses}

A problem of importance both theoretically and in practice is that spectral risks inherit the sensitivity of CVaR to (unbounded) heavy-tailed losses \citep{bhat2019a}. This means that naive empirical estimates have extremely high variance, and the previously-cited concentration results \citep{pandey2019a,khim2020a} no longer hold. In recent years, an active line of research has studied the problem of designing algorithms with near-optimal guarantees (in terms of the traditional risk) under heavy-tailed losses; in our section \ref{sec:theory_highprob}, we show how for an important sub-class of spectral risk tasks, we can utilize standard confidence boosting techniques \citep{holland2021c}, integrating them with the Algorithm \ref{algo:smd} to obtain high-probability guarantees for a procedure that does not use first-order information, and admits heavy-tailed loss distributions.

\section{Preliminaries}

\subsection{Setup}

Denoting the underlying data space by $\ZZ$, we denote by $\loss:\RR^{d} \times \ZZ \to \RR$ a generic loss function, assumed to satisfy $\loss(w;z) \in \RR$ for all $w \in \RR^{d}$ and $z \in \ZZ$. Our general-purpose random data is $Z \sim \ddist$, and the resulting random loss values $\loss(w;Z)$ have a distribution function denoted by $\dfun_w(u) \defeq \ddist\{\loss(w;Z) \leq u\}$, for all $w \in \RR^{d}$ and $u \in \RR$. When we make use of data-driven estimates of the distribution function, we shall denote this by $\dfunhat$. For indexing purposes, we write $[k] \defeq \{1,\ldots,k\}$ for any positive integer $k$. For any sequence $(U_t)$ of random objects, we shall denote sub-sequences by $U_{[t]} \defeq (U_1,\ldots,U_t)$.

The traditional choice of \term{risk function} in loss-driven machine learning tasks is the expected value. Written explicitly, this is
\begin{align}
\label{eqn:risk_defn}
\risk(w) \defeq \exx_\ddist \loss(w;Z) \defeq \int_\ZZ \loss(w;z)\,\ddist(\dif z).
\end{align}
Another important risk function is the \term{conditional value at risk}, defined for $\beta \in [0, 1)$ by
\begin{align}
\label{eqn:crisk_defn}
\crisk_{\beta}(w) & \defeq \frac{1}{1-\beta} \int_{u}^{1} \vrisk_{u}(w) \, \dif u\\
\nonumber
& = \exx_\ddist \loss(w;Z) \idc_{\{ \loss(w;Z) \geq \vrisk_{\beta}(w) \}},
\end{align}
where $\vrisk_{\beta}(w) \defeq \inf\left\{ u : \dfun_w(u) \geq \beta \right\}$, the $\beta$-level quantile of $\loss(w;Z)$. In this work, our focus will be on a class of risk functions which can be given in terms of $\vrisk_{\beta}$ modulated by a user-specified density function. More concretely, let $\spec:[0,1] \to \RR_{+}$ be a non-negative, non-decreasing function that integrates to $1$. We then define the \term{spectral risk} of $w$ induced by $\spec$ as
\begin{align}
\label{eqn:srisk_defn}
\srisk(w) \defeq \int_{0}^{1} \vrisk_{\beta}(w) \spec(\beta)\,\dif\beta.
\end{align}
From the definition (\ref{eqn:crisk_defn}) of $\crisk$, we see that setting $\spec(u) = \idc_{\{\beta < u \leq 1\}} / (1-\beta)$, one recovers the special case of $\srisk(w) = \crisk_{\beta}(w)$. A direct attack on $\srisk$ presents difficulties, in particular with respect to computing first-order (stochastic) estimates that might in principle drive an iterative learning algorithm. In the vein of alleviating such difficulties, using insights going back to the influential work of \citet{flaxman2004a}, we introduce the \term{smoothed spectral risk}
\begin{align}
\label{eqn:ssrisk_defn}
\ssrisk(w) \defeq \exx_\ndist \left[ \srisk\left( w + \pert U \right) \right],
\end{align}
where $U \sim \ndist$ is uniformly distributed over the unit ball $\{u \in \RR^{d}: \|u\| \leq 1\}$, and the parameter controlling the degree of shift satisfies $0 < \pert < 1$.

\subsection{Basic properties}

As long as the loss distribution has a positive density, spectral risks can be expressed in a more convenient form, as follows.
\begin{lem}\label{lem:srisk_can_estimate}
Let $\dfun_w$ be invertible and differentiable for $w \in \RR^{d}$. Then, we have
\begin{align}
\label{eqn:srisk_can_estimate}
\srisk(w) = \exx_\ddist \loss(w;Z) \spec\left( \dfun_w(\loss(w;Z)) \right),
\end{align}
where $\srisk$ is the spectral risk defined in (\ref{eqn:srisk_defn}).
\end{lem}
\begin{proof}
Since $\dfun_w$ is invertible and continuous, we have $\vrisk_{\beta}(w) = \dfun_{w}^{-1}(\beta)$ for any $w \in \RR^{d}$ and $0 < \beta < 1$. We then see that
\begin{align*}
\int_{0}^{1} \dfun_{w}^{-1}(\beta) \spec(\beta) \, \dif\beta = \int_{-\infty}^{\infty} \dfun_w^{-1}(\dfun_w(u)) \spec(\dfun_w(u)) \dfun_w^{\prime}(u) \, \dif u = \int_{-\infty}^{\infty} u \spec(\dfun_w(u)) \, \dfun_w(\dif u),
\end{align*}
noting that the first equality uses integration by substitution. The right-most expression is none other than $\exx_{\ddist} \loss(w;Z) \spec(\dfun_w(\loss(w;Z)))$.
\end{proof}
\noindent With Lemma \ref{lem:srisk_can_estimate} as context, the following stochastic estimators will be of interest:
\begin{align}
\label{eqn:sloss_defn}
\sloss(w;Z) & \defeq \loss(w;Z) \spec\left( \dfun_w(\loss(w;Z)) \right)\\
\label{eqn:slhat_defn}
\slhat(w;Z) & \defeq \loss(w;Z) \spec\left( \dfunhat_w(\loss(w;Z)) \right).
\end{align}
If the distribution function $\dfun_w$ were known, then access to a random sample of $Z \sim \ddist$ would immediately imply access to an unbiased estimator of $\srisk$. Unfortunately, in practice $\dfun_w$ will never be known, and thus must be estimated based on observable data. We are denoting this empirical estimator as $\dfunhat_w$. Since $\slhat$ can be computed based on observable data, it will play a central role in the algorithms we study in the following section.

The following lemma gives a useful representation of any spectral risk in terms of CVaR, which using linearity of the integral lets us inherit some useful properties of the latter.
\begin{lem}[\citenum{shapiro2013a}, Rmk.~3, Eqn.~42]
For the spectral risk $\srisk$ given by (\ref{eqn:srisk_defn}), we can write
\begin{align}
\label{eqn:srisk_cvar}
\srisk(w) = \int_{0}^{1} \crisk_{\beta}(w)\,\mu_{\spec}(\dif\beta), \quad w \in \RR^{d}
\end{align}
where $\mu_{\spec}$ is a measure on the unit interval that does not depend on $w$.
\end{lem}
Introducing the smoothed risk $\ssrisk$ is only going to be fruitful if it is easier to optimize than the original non-smooth risk. Fortunately, as the following result shows, it is straightforward to obtain unbiased first-order information for the smoothed risk.
\begin{lem}\label{lem:ssrisk_grad_link}
In contrast with $\ndist$ used in definition (\ref{eqn:ssrisk_defn}), let $\ndist_1$ denote the uniform distribution over the unit sphere $\{u \in \RR^{d}: \|u\|=1\}$, taking random direction $U \sim \ndist_1$, we have
\begin{align}
\label{eqn:ssrisk_grad_link}
\frac{d}{\pert} \exx_{\ndist_1}\left[ \srisk(w + \pert U) U \right] = \nabla\ssrisk(w)
\end{align}
for any $w \in \RR^{d}$ and $0 < \pert < 1$.
\end{lem}
\begin{proof}
Follows from \citet[Lem.~1]{flaxman2004a}, using our smoothed risk (\ref{eqn:ssrisk_defn}).
\end{proof}

\section{Guarantees in expectation on $\RR^{d}$}\label{sec:theory_expectation}

In this section, we specify a concrete learning algorithm, and seek excess spectral risk bounds in expectation. This procedure and its guarantees will act as a key building block to be utilized in the following section.

\paragraph{Learning algorithm}

We essentially consider a stochastic mirror descent update, with first-order estimates using the form suggested by Lemmas \ref{lem:srisk_can_estimate} and \ref{lem:ssrisk_grad_link}. Making this more explicit, let $\Phi: \RR^{d} \to \RR$ be a strictly convex function, and let $\breg$ denote the Bregman divergence induced by $\Phi$. Modulated by positive step sizes $(\alpha_t)$, We generate a sequence of iterates $(w_t)$ using the following update rule:
\begin{align}
\label{eqn:smd_actual}
w_{t+1} = \argmin_{w \in \WW} \left[ \langle \widehat{G}_t, w \rangle + \frac{1}{\alpha_t}  \breg(w;w_t) \right].
\end{align}
The key stochastic ``gradients'' used here are defined as
\begin{align}
\label{eqn:grad_noisy}
\widehat{G}_t \defeq \frac{d}{\pert} \slhat\left(w_t + \pert U_t;Z_t\right) U_t,
\end{align}
where underlying sequences $(U_t)$ and $(Z_t)$ are assumed to be iid, with $U_t \sim \ndist_1$ and $Z_t \sim \ddist$ for all integer $t>0$, and $\slhat$ is as defined in (\ref{eqn:slhat_defn}). The full procedure is summarized in Algorithm \ref{algo:smd}.
\begin{algorithm}[t!]
\caption{Derivative-free stochastic mirror descent under spectral risks.}
\label{algo:smd}
\begin{algorithmic}
\State \textbf{inputs:} initial point $w_0 \in \WW$, step sizes $(\alpha_t)$, data set size $M$, and max iterations $T$.
\medskip
\For{$t \in \{0,\ldots,T-1\}$}
 \State Get ancillary sample $\{Z_{t,1}^{\prime},\ldots,Z_{t,M}^{\prime}\}$, setting $\dfunhat_{w_t}(u) \defeq (\nicefrac{1}{M})\sum_{i=1}^{M} \idc_{\{ \loss(w_t;Z_{t,i}^{\prime}) \leq u \}}$.
 \State Sample $U_t$ and $Z_t$, compute gradient $\widehat{G}_t$ via (\ref{eqn:grad_noisy}).
 \smallskip
 \State Update $w_t \mapsto w_{t+1}$ via (\ref{eqn:smd_actual}).
 \smallskip
\EndFor
\State \textbf{return:} $\overbar{w}_{T} \defeq (\nicefrac{1}{T})\sum_{t=1}^{T}w_t$.
\end{algorithmic}
\end{algorithm}

\paragraph{Technical conditions}

Letting $\WW \subset \RR^{d}$ be closed, bounded, and convex, denote any minimizer of $\srisk$ over $\WW$ by $w^{\ast}$. Denote the diameter of $\WW$, measured respectively using the underlying norm $\|\cdot\|$ and the Bregman divergence $\breg$, as $\diameter \defeq \sup\{\|w-w^{\prime}\|: w,w^{\prime} \in \WW\}$ and $\diameter_{\Phi} \defeq \sup\{\breg(w;w^{\prime}): w,w^{\prime} \in \WW\}$. Since the random perturbations may take us to points outside $\WW$, let us define $\CC \defeq \{w+u: w \in \WW, \|u\| \leq 1\}$ to cover all such possibilities. Let $\Phi$ be $\strong$-strongly convex on $\CC$ (e.g., $\Phi(u)=\|u\|_{2}^{2}/2$, with $\strong=1$). On the underlying loss distribution, we assume the following moment bounds are finite:
\begin{align*}
\smooth_{\risk} & \defeq \sup_{v \in \CC} \risk(v)\\
s_1^{2} & \defeq \sup_{v \in \CC} \exx_{\ndist_1,\ddist}\left[ \|\sloss(v;Z)U - \exx_{\ndist_1,\ddist}\left[\sloss(v;Z)U\right]\|^{2} \right]\\
s_2^{2} & \defeq \sup_{v \in \CC} \exx_{\ddist}|\loss(v;Z)|^{2}.
\end{align*}
Finally, we assume that the conditions of Lemmas \ref{lem:srisk_can_estimate}--\ref{lem:ssrisk_grad_link} hold, the loss is such that $w \mapsto \loss(w;Z)$ is convex and continuous on $\CC$, and that the spectral density $\spec(\cdot)$ is $\smooth_{\spec}$-Lipschitz.

\begin{thm}[Spectral risk bounds in expectation]\label{thm:smd_srisk}
Under the preceding assumptions, let $\overbar{w}_{T}$ be the output of Algorithm \ref{algo:smd} run for $T$ steps, using $M$ points for distribution estimates, and step sizes $\alpha_t = \strong/(\smooth_{\risk} + 1/c_T)$ with $c_T \defeq (\pert/d)\sqrt{2\diameter_{\Phi} \strong/(T(s_1^{2}+(\smooth_{\spec}s_2)^{2}))}$, fixed for all $t$. Then we have
\begin{align*}
\exx\left[ \srisk(\overbar{w}_T)-\srisk(w^{\ast}) \right] \leq 2\smooth_{\risk}\pert + \frac{d}{\pert}\left[ \sqrt{\frac{2\diameter_{\Phi}(s_1^{2} + (\smooth_{\spec}s_2)^{2})}{T\strong}} + \frac{\smooth_{\risk}\diameter_{\Phi}}{T\strong} + \smooth_{\spec}\smooth_{\risk}\diameter\sqrt{\frac{\pi}{2M}} \right]
\end{align*}
for any choice of $0 < \pert < 1$, where expectation is taken over $U_{[T]}$, $Z_{[T]}$, and the ancillary data.
\end{thm}
The proof of Theorem \ref{thm:smd_srisk} is composed of several simple steps, but due to its length, we relegate the proof details of this and subsequent results to the supplementary appendix.

\paragraph{Sample complexity}

The guarantee given by Theorem \ref{thm:smd_srisk} is quite general, since the parameters $\pert$, $T$, and $M$ are free to be set as desired. Let us consider the important situation in which we are constrained to at most $n$ iid samples from the data distribution $\ddist$. In running Algorithm \ref{algo:smd}, for a simple and concrete example, let us set $M=\lceil \sqrt{n} \rceil$ to specify a precision level. Since each step uses $M + 1$ points, the number of steps $T$ can thus be no greater than $n/(1+\lceil \sqrt{n} \rceil)$, and setting $T = \lfloor n/(1+\lceil \sqrt{n} \rceil) \rfloor$ we will always be on budget, i.e., $T(M+1) \leq n$. Plugging these values in for $T$ and $M$, and subsequently minimizing the bound from Theorem \ref{thm:smd_srisk} as a function of $\pert$, we get $\exx\left[ \srisk(\overbar{w}_T)-\srisk(w^{\ast}) \right] \leq \varepsilon_{1}(n)$, where we define
\begin{align}\label{eqn:sample_complexity}
\varepsilon_{1}(n) \defeq 2\sqrt{2\smooth_{\risk}d\left[ \sqrt{\frac{2\diameter_{\Phi}(s_1^{2} + (\smooth_{\spec}s_2)^{2})}{\lfloor n/(1+\lceil \sqrt{n} \rceil) \rfloor\strong}} + \frac{\smooth_{\risk}\diameter_{\Phi}}{\lfloor n/(1+\lceil \sqrt{n} \rceil) \rfloor\strong} + \smooth_{\spec}\smooth_{\risk}\diameter\sqrt{\frac{\pi}{2\lceil \sqrt{n} \rceil}} \right]},
\end{align}
and thus to achieve $\exx\left[ \srisk(\overbar{w}_T)-\srisk(w^{\ast}) \right] \leq \epsilon$, the sample complexity is $\OO(\epsilon^{-8})$.

\paragraph{Comparison with derivative-free literature}

Here we try to place the sample complexity derived from (\ref{eqn:sample_complexity}) into some context. For a convex objective, the main result of \citet[Thm.~1]{flaxman2004a} yields a sample complexity of $\OO(\epsilon^{-6})$ for a derivative-free update analogous to the one used here. The reason for the slower $\OO(\epsilon^{-8})$ rate here is clear: while \citet{flaxman2004a} consider traditional risks, for our setup using spectral risks, we allocate (most) data to estimate $\dfun_{w_t}$ at each step, a requirement that does not arise in the traditional setting. To obtain faster rates, there are several natural routes. First, one could optimize the bound in Theorem \ref{thm:smd_srisk} with respect to ancillary dataset size $M$; we set $M = \sqrt{n}$ here for simplicity and readability. Second, our choice of the gradient estimator (\ref{eqn:grad_noisy}) was to maximize the ease of exposition; many alternative approaches have been studied over the past decade \citep{saha2011a,belloni2015a,gasnikov2017a,balasubramanian2021a}, and can readily be adapted to our problem setting to further improve the sample complexity; see \citet[Sec.~4.2]{larson2019a} for a survey of relevant methods.
\begin{rmk}[Faster rates for CVaR]
Our Lipschitz assumption on $\spec$ in Theorem \ref{thm:smd_srisk} precludes CVaR from the class of risks for which the performance guarantee holds. We consider this justifiable, since sub-gradient information is easily computed for the special case of CVaR, and $\OO(\epsilon^{-2})$ rates have already been proved in that restricted setting for stochastic sub-gradient algorithms \citep{soma2020a,holland2021c}.
\end{rmk}

\section{High-probability guarantees for unbounded losses}\label{sec:theory_highprob}

Theorem \ref{thm:smd_srisk} only provides guarantees in expectation, and thus it is natural to consider the output of Algorithm \ref{algo:smd} as an inexpensive but ``weak'' candidate. If we split up the data, obtaining multiple weak candidates and setting aside some data for careful validation, then we can apply a robust confidence-boosting technique, as follows.

If $n$ is our budget for sampling from $\ddist$, and we want $k$ independent candidates, run Algorithm \ref{algo:smd} $k$ times independently, using $\lfloor n/(k+1) \rfloor$ points each time. Denote the output of these sub-processes by $\overbar{w}^{(1)},\ldots,\overbar{w}^{(k)}$. Having computed these, we still have $\lfloor n/(k+1) \rfloor$ points left, and this data will be needed to determine which of the $k$ candidates to use. One half of this remaining data is used to construct $\dfunhat_w$, distinct from the estimates used within Algorithm \ref{algo:smd} to get each $\overbar{w}^{(j)}$. The other half, denoted $Z_{i}^{\prime\prime}$ for $i = 1,\ldots,\lfloor n/(k+1) \rfloor/2$, is used to compute a robust location estimate. As a concrete example, for each $j$ compute
\begin{align}
\widehat{\risk}_{\spec}^{(j)} \defeq \argmin_{a \in \RR} \sum_{i = 1}^{\lfloor n/(k+1) \rfloor/2} \rho\left(\frac{a - \loss(\overbar{w}^{(j)};Z_{i}^{\prime\prime})\spec(\dfunhat_{i,j}) }{b}\right),
\end{align}
where we have set $\dfunhat_{i,j} \defeq \dfunhat_{\overbar{w}^{(j)}}(\loss(\overbar{w}^{(j)};Z_{i}^{\prime\prime}))$ for readability, $\rho$ is a differentiable strictly convex function, and $b > 0$ is a scaling parameter. For an appropriate choice of $\rho$ and $b$, this is an M-estimator of the spectral risk incurred by $\overbar{w}^{(j)}$ \citep{catoni2012a,devroye2016a}. For each $j$, we introduce the key intermediate quantity
\begin{align}
\overbar{\risk}_{\spec}^{(j)} \defeq \exx_{\ddist}\left[ \loss(\overbar{w}^{(j)};Z)\spec(\dfunhat_{\overbar{w}^{(j)}}(\loss(\overbar{w}^{(j)};Z))) \right].
\end{align}
For comparison, let write $\risk_{\spec}^{(j)} \defeq \srisk(\overbar{w}^{(j)})$ for the spectral risk incurred by the $j$th candidate. Assuming the spectral density is bounded as $\spec(\cdot) \leq \overbar{\spec} < \infty$, then we can obtain the following upper bounds:
\begin{align}
\nonumber
|\widehat{\risk}_{\spec}^{(j)} - \risk_{\spec}^{(j)}| & \leq | \widehat{\risk}_{\spec}^{(j)} - \overbar{\risk}_{\spec}^{(j)} | + | \overbar{\risk}_{\spec}^{(j)} - \risk_{\spec}^{(j)} |\\
\nonumber
& \leq | \widehat{\risk}_{\spec}^{(j)} - \overbar{\risk}_{\spec}^{(j)} | + \smooth_{\spec} \exx_{\ddist}|\loss(\overbar{w}^{(j)};Z)| \left[\sup_{u \in \RR}|\dfunhat_{\overbar{w}^{(j)}}(u) - \dfun_{\overbar{w}^{(j)}}(u)|\right]\\
\label{eqn:highprob_helper}
& \leq \varepsilon_{2}(n;k,\delta) \defeq 2\overbar{\spec}s_{2}\sqrt{\frac{2(1+\log(2\delta^{-1}))}{\lfloor n/(k+1) \rfloor}} + \smooth_{\spec} s_{2} \sqrt{\frac{\log(4\delta^{-1})}{\lfloor n/(k+1) \rfloor}},
\end{align}
where (\ref{eqn:highprob_helper}) holds with probability no less than $1-\delta$, over the random draw of the data points used to compute $\dfunhat_w$ and $\widehat{\risk}_{\spec}^{(j)}$ here, conditioned on $\overbar{w}^{(j)}$ (detailed proof in the appendix). Algorithmically, all we need to do is choose the best candidate based on the above robust estimates, namely
\begin{align}\label{eqn:highprob_candidate}
\overbar{w}^{\ast} \defeq \overbar{w}^{(\star)}, \text{ where } \star \defeq \argmin_{j \in [k]} \widehat{\risk}_{\spec}^{(j)}.
\end{align}
This ``boosted'' choice enjoys a high-probability guarantee, as desired.
\begin{thm}\label{thm:highprob}
For confidence parameter $0 < \delta < 1$, if we set the number of weak candidates to $k = \lceil \log(2\lceil\log(\delta^{-1})\rceil) \rceil$ and compute $\overbar{w}^{\ast}$ as in (\ref{eqn:highprob_candidate}), then we have
\begin{align*}
& \srisk(\overbar{w}^{\ast}) - \srisk(w^{\ast}) \leq \mathrm{e}\,\varepsilon_{1}\left(\frac{n}{k+1}\right) + 2\varepsilon_{2}(n;k,\delta)
\end{align*}
with probability no less than $1-3\delta$, where $\varepsilon_{1}$ and $\varepsilon_{2}$ are as defined in (\ref{eqn:sample_complexity}) and (\ref{eqn:highprob_helper}).
\end{thm}

\section{Fast implementation and empirical analysis}\label{sec:empirical}

The procedure outlined by Algorithm \ref{algo:smd} yields clear formal guarantees for a wide class of spectral risks where exact gradient computations are infeasible, as described in Theorems \ref{thm:smd_srisk} and \ref{thm:highprob}. On the other hand, in the interest of practical utility, we would like to improve the slow convergence rates using even approximate first-order information, since in many cases both $\loss$ and $\spec$ will be at least sub-differentiable. In this section, we outline a simple modified procedure which makes more direct use of the first-order information we have, empirically comparing it with both Algorithm \ref{algo:smd} and traditional ERM, as natural benchmarks.

\paragraph{Modified procedure}

Issues with differentiability arise chiefly because the form of $\dfun_{w}$ is unknown. Arguably the simplest way to circumvent this difficulty is to introduce a parametric model to approximate the loss CDF. Our modified procedure takes Algorithm \ref{algo:smd} as a starting point, and makes the following changes. First, at each step in the main loop, instead of $\dfunhat_{w_t}$, we use a folded Normal distribution, with mean and standard deviation parameters set using empirical estimates based on the ancillary sample, evaluated at $w_t$. Denote this parametric estimate of $\dfun_{w_t}$ by $\dfunhat_t$, and its derivative by $\ddfunhat_t$. Next, conditioned on $\dfunhat_t$, a few applications of the chain rule lets us compute the partial derivatives of $w \mapsto \loss(w;Z)\spec(\dfunhat_t(\loss(w;Z)))$ easily. At each step $t$ we will use the following gradient estimate:
\begin{align}
\label{eqn:grad_noisy_fast}
\widetilde{G}_t \defeq \left[ \spec(\dfunhat_t(\loss_t)) + \loss_t \spec^{\prime}(\dfunhat_t(\loss_t))\ddfunhat_t(\loss_t) \right] \nabla \loss(w_t;Z_t),
\end{align}
where we have written $\loss_t \defeq \loss(w_t;Z_t)$ for readability. Our modified procedure is completed by using the update (\ref{eqn:smd_actual}), replacing $\widehat{G}_t$ with $\widetilde{G}_t$ just specified.

\paragraph{Experimental design}

We compare three methods: derivative-free Algorithm \ref{algo:smd} (called \texttt{default} in the figures), the modified procedure described in the previous paragraph (called \texttt{fast}), and traditional empirical risk minimization (called \texttt{off}). We mean ``traditional'' in terms of the risk, and thus \texttt{off} amounts to running (\ref{eqn:smd_actual}) and simply replacing $\widehat{G}_t$ with the original loss gradient $\nabla\loss(w_t;Z_t)$, and using all data for training (no ancillary set needed). All methods are run using the Euclidean norm for distance computation, and thus (\ref{eqn:smd_actual}) is just a standard steepest descent update with step size $\alpha_t$, plug projection onto $\WW$. We apply each of these methods to classification tasks on a number of standard benchmark datasets, using standard multi-class logistic regression. For \texttt{default}, we fix $\alpha_t = 2\gamma/(d\sqrt{n})$, where $d$ is the total number of parameters to be determined, and $n$ is the number of training samples. The extra factor $\gamma/d$ is to account for the coefficient in (\ref{eqn:grad_noisy}); this approach mirrors other derivative-free procedures \citep[Thm.~1]{flaxman2004a}. For \texttt{fast} and \texttt{off}, we simply fix $\alpha_t = 2/\sqrt{n}$. These settings were selected before running any experiments. For each dataset, 10 independent trials are run, in which the full dataset is randomly shuffled before starting, with each method randomly initialized to the same point, and run for 50 epochs. Finally, as an illustrative example for our tests, we set $\spec(\cdot)$ to the exponential risk spectrum $\spec(u) = c\,\mathrm{e}^{-c(1-u)}/(1-\mathrm{e}^{-c})$, fixing $c=1$, a well-established standard from the literature \citep{dowd2006a,pandey2019a}.

\paragraph{Software}

All code required to pre-process the data, run the experiments, and re-create the figures in this paper is available at: \url{https://github.com/feedbackward/spectral}

\paragraph{Results and discussion}

Plots of empirical spectral risk and misclassification rates are shown in Figures \ref{fig:empirical}--\ref{fig:empirical2}. The plotted trajectories represent averages taken over all trials, and the shaded area around the \emph{test} error is the average $\pm$ standard deviation. The individual plot titles (e.g., \texttt{cifar10}) refer to the datasets used. Additional data details are given in the appendix. One key observation that can be made is that the proposed modification \texttt{fast} achieves an appealing balance of performance in terms of spectral risk and misclassification error. Depending on the dataset, we see that without tuning the step size parameter, \texttt{off} may outperform \texttt{fast} in terms of the spectral risk, though the only stark difference appears in the case of \texttt{cifar10}, and additional testing has shown this can be mitigated with more careful step size setting. That said, it is quite remarkable that \texttt{fast} maintains a superior misclassification rate across all datasets tested. On the other hand, as suggested by the results of section \ref{sec:theory_expectation}, \texttt{default} is slow to converge, and also quite sensitive to step size settings. For simplicity and transparency we have used a fixed step size for each method, and though it should be noted that dataset-specific tuning of the step size does allow us to ensure \texttt{default} converges at close to the expected rate, the clear differences in sensitivity and speed make \texttt{fast} the first choice for practical spectral risk-based learning tasks. To further refine \texttt{default}, introducing multi-point derivative-free methods and update rules that better utilize sparse inputs \citep{balasubramanian2021a}. As for \texttt{fast}, since the current model is quite naive with respect to the form of the loss distribution, introducing more robust modeling techniques \citep{lange1989a} is expected to have a major impact on practical utility.

\begin{figure}[t]
\centering
\includegraphics[width=0.35\textwidth]{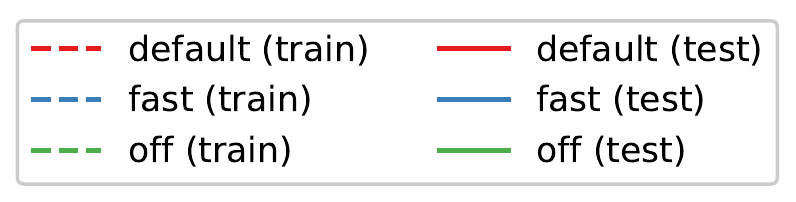}\\
\includegraphics[width=0.25\textwidth]{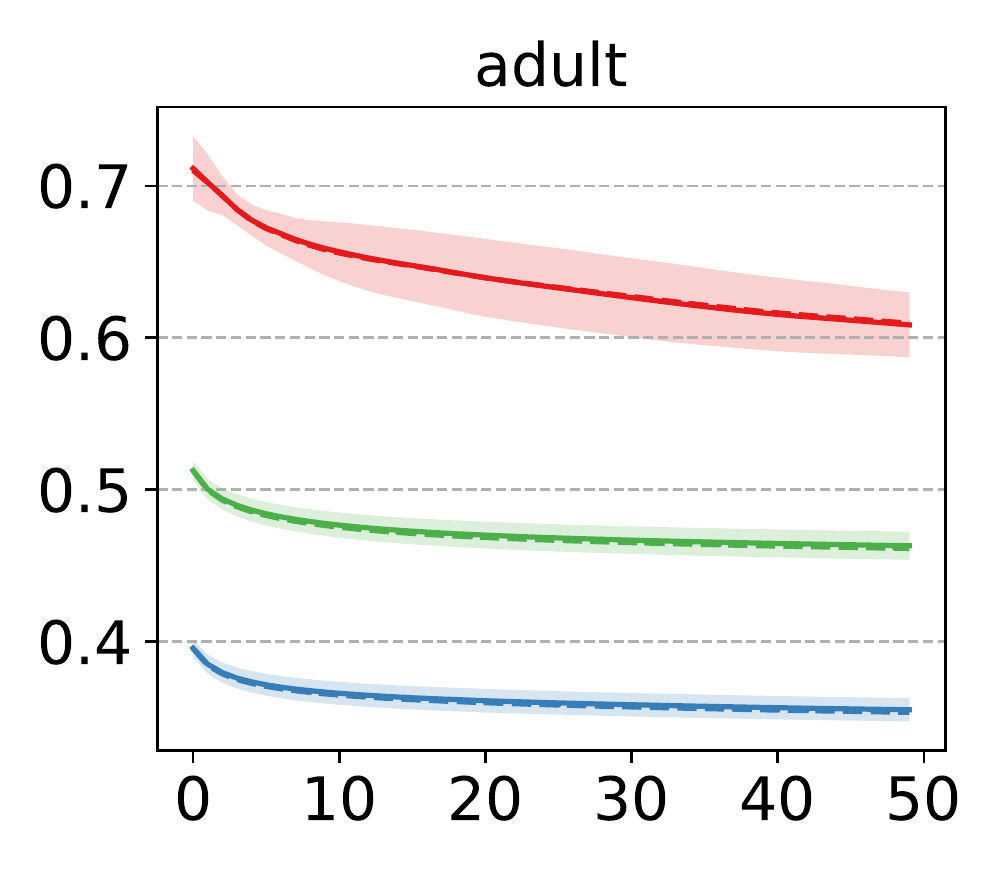}\includegraphics[width=0.25\textwidth]{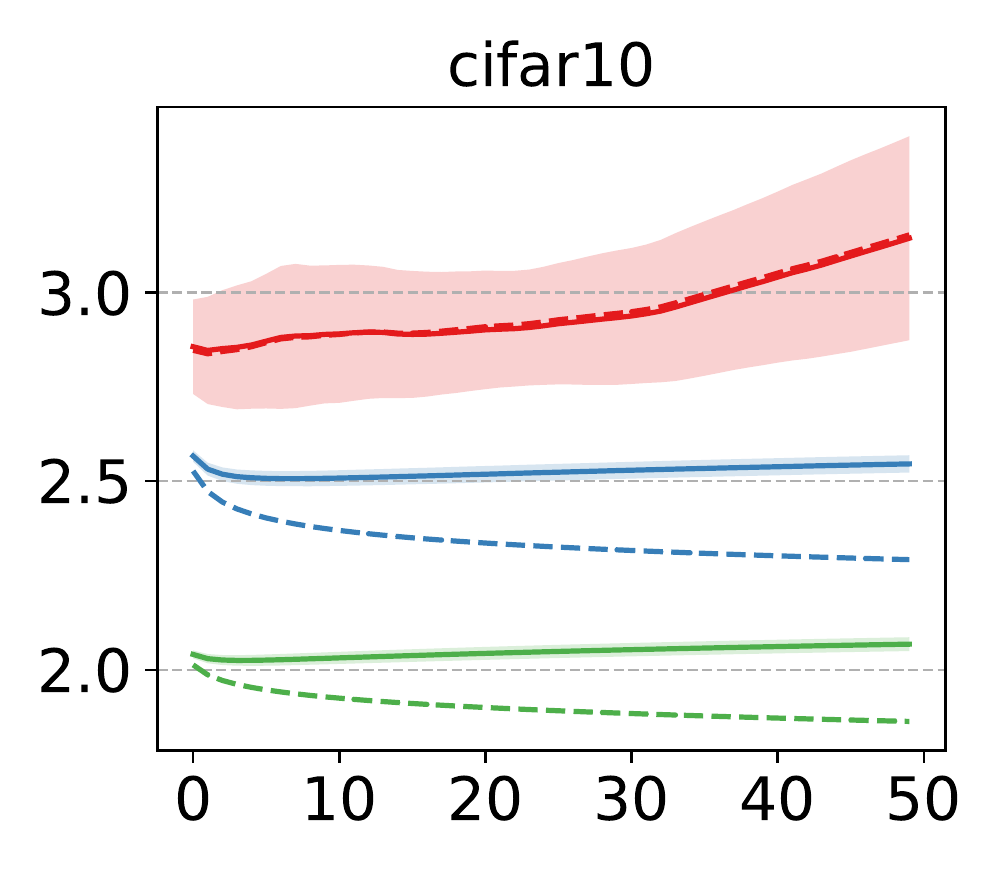}\includegraphics[width=0.25\textwidth]{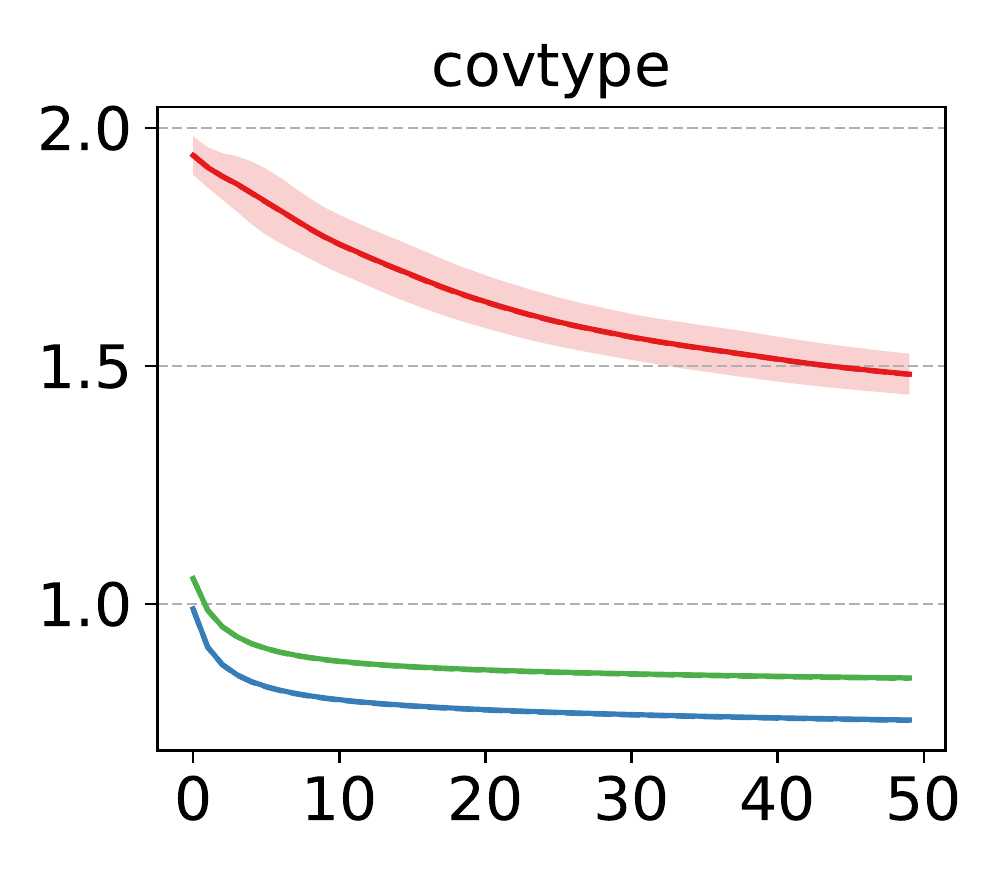}\includegraphics[width=0.25\textwidth]{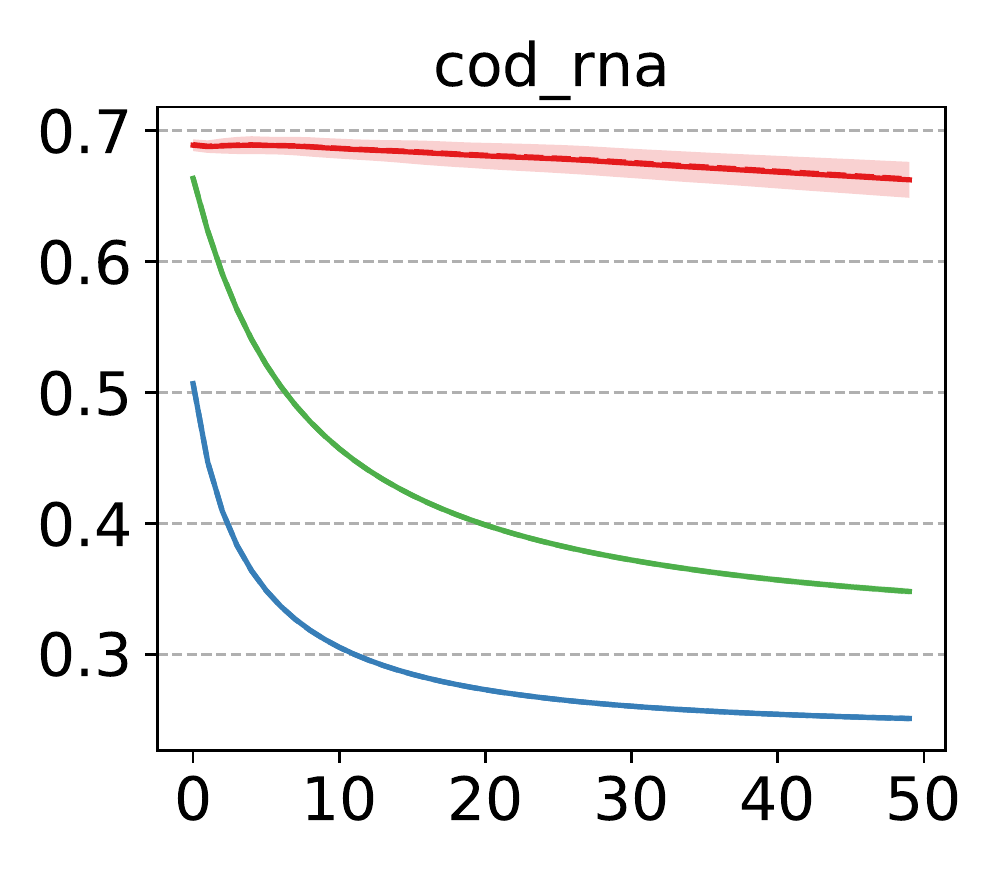}\\
\includegraphics[width=0.25\textwidth]{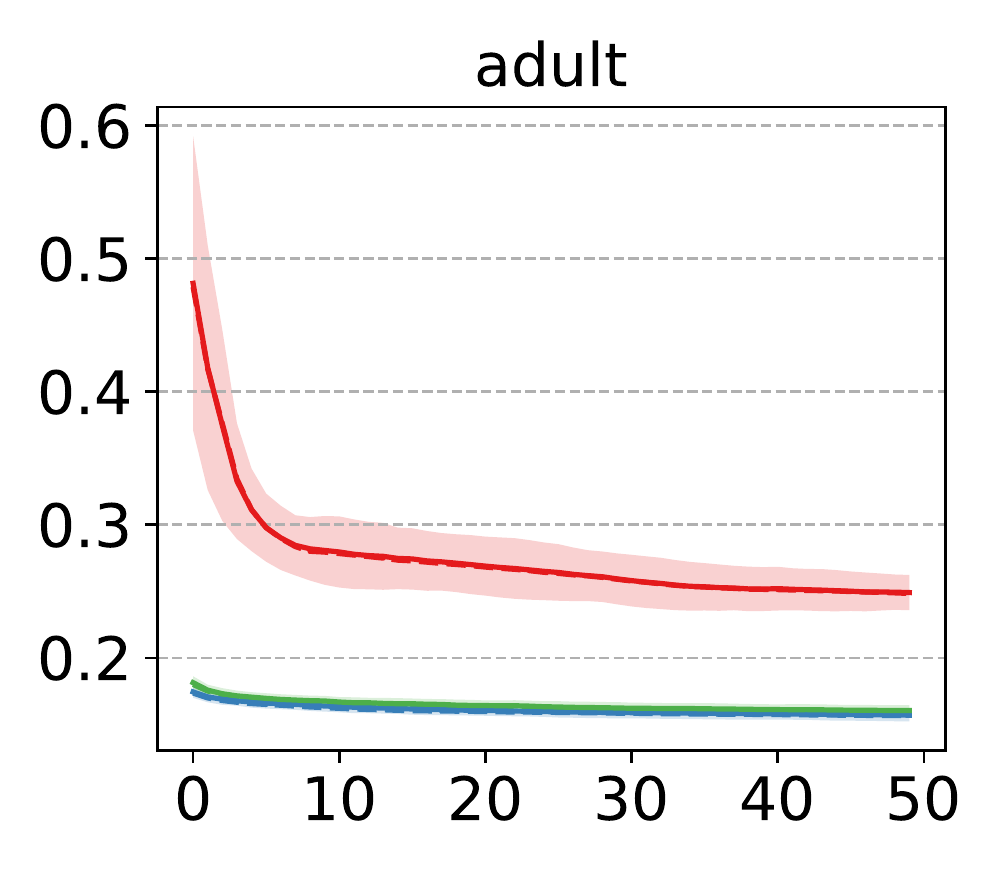}\includegraphics[width=0.25\textwidth]{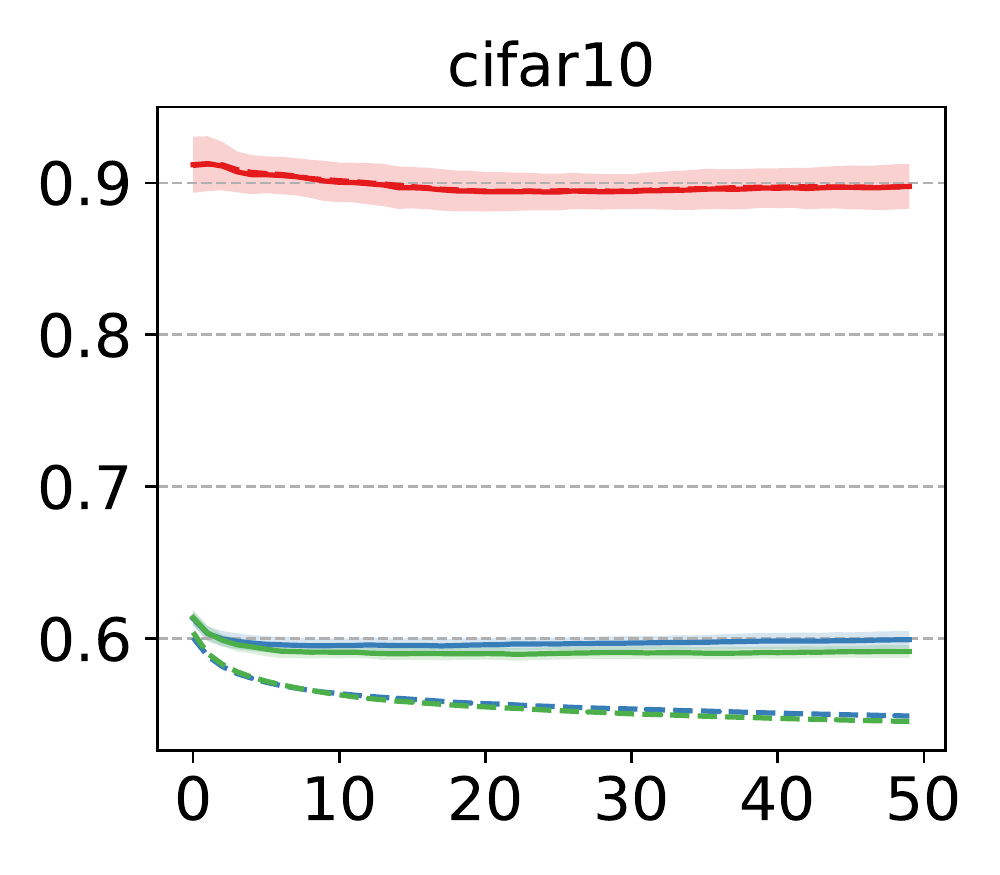}\includegraphics[width=0.25\textwidth]{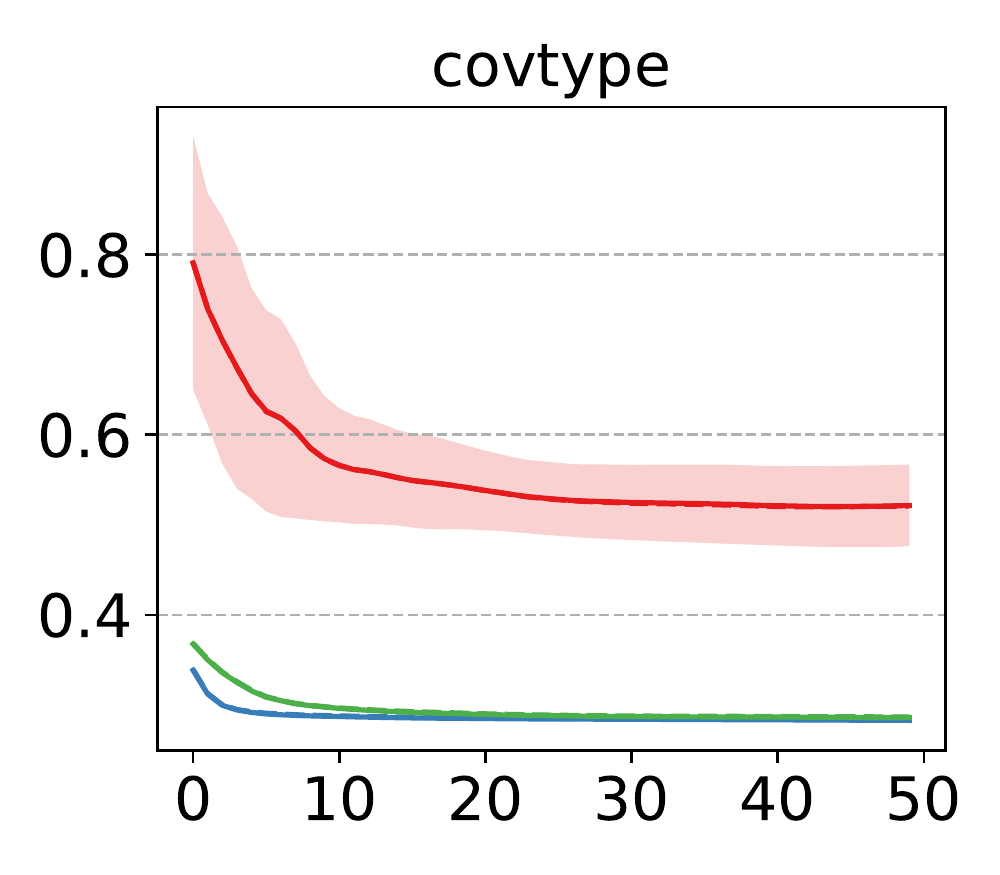}\includegraphics[width=0.25\textwidth]{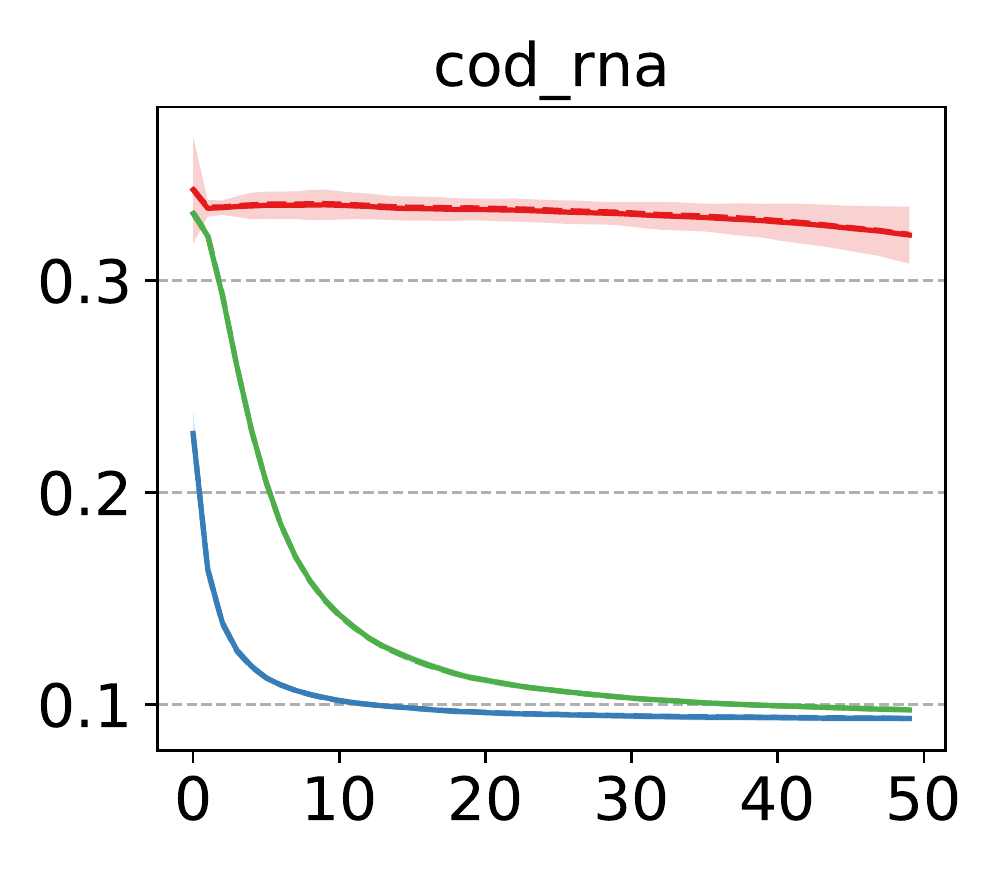}
\caption{Error trajectories for each method and dataset, on both training (dashed) and testing (solid) data. Top row: empirical spectral risks. Bottom row: misclassification rates.}
\label{fig:empirical}
\end{figure}
\begin{figure}[t]
\centering
\includegraphics[width=0.25\textwidth]{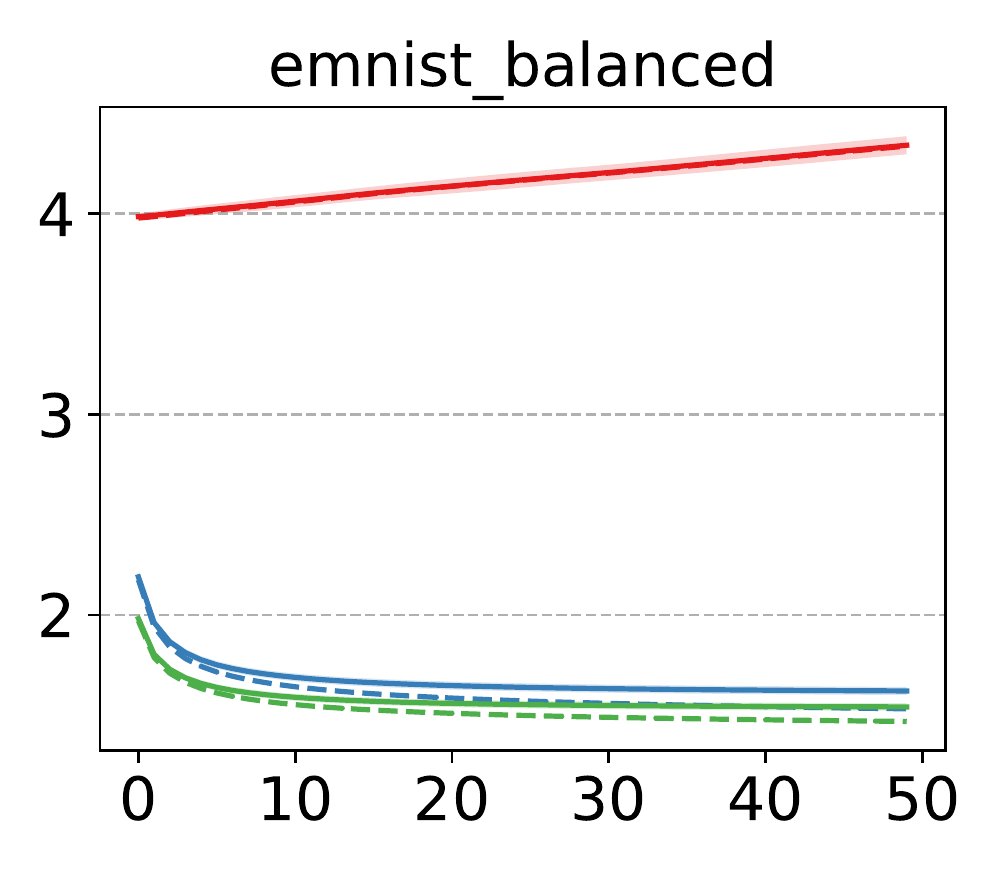}\includegraphics[width=0.25\textwidth]{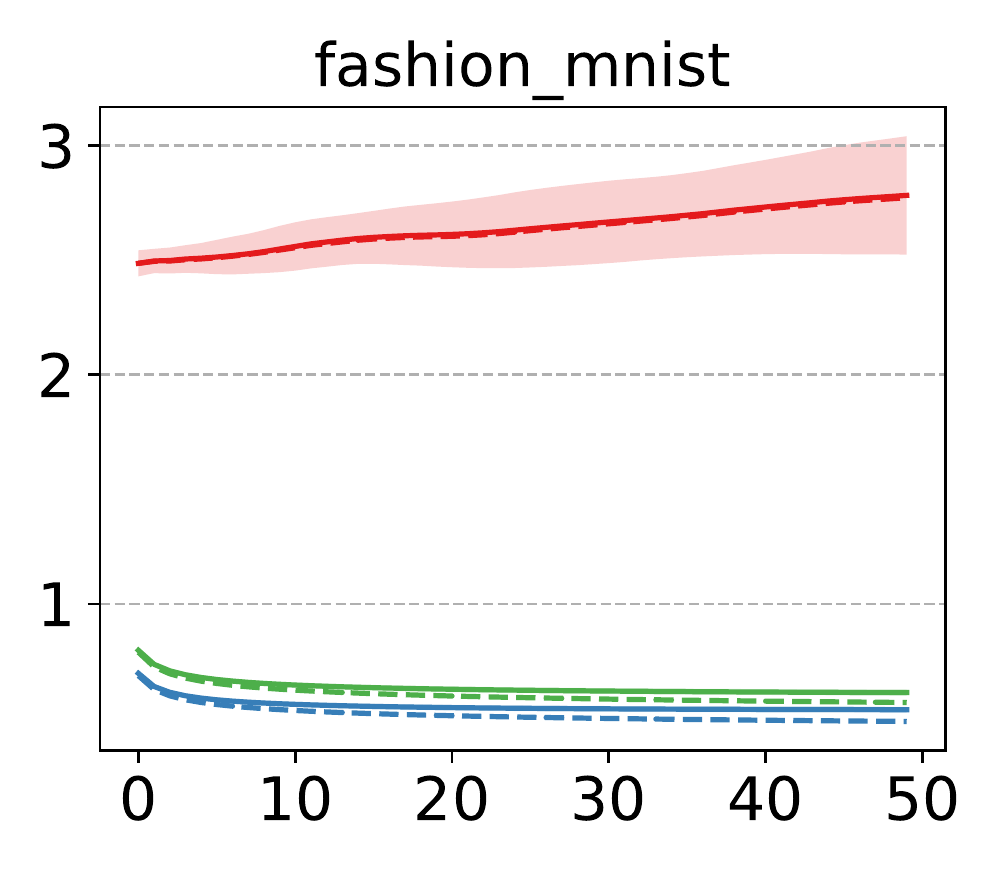}\includegraphics[width=0.25\textwidth]{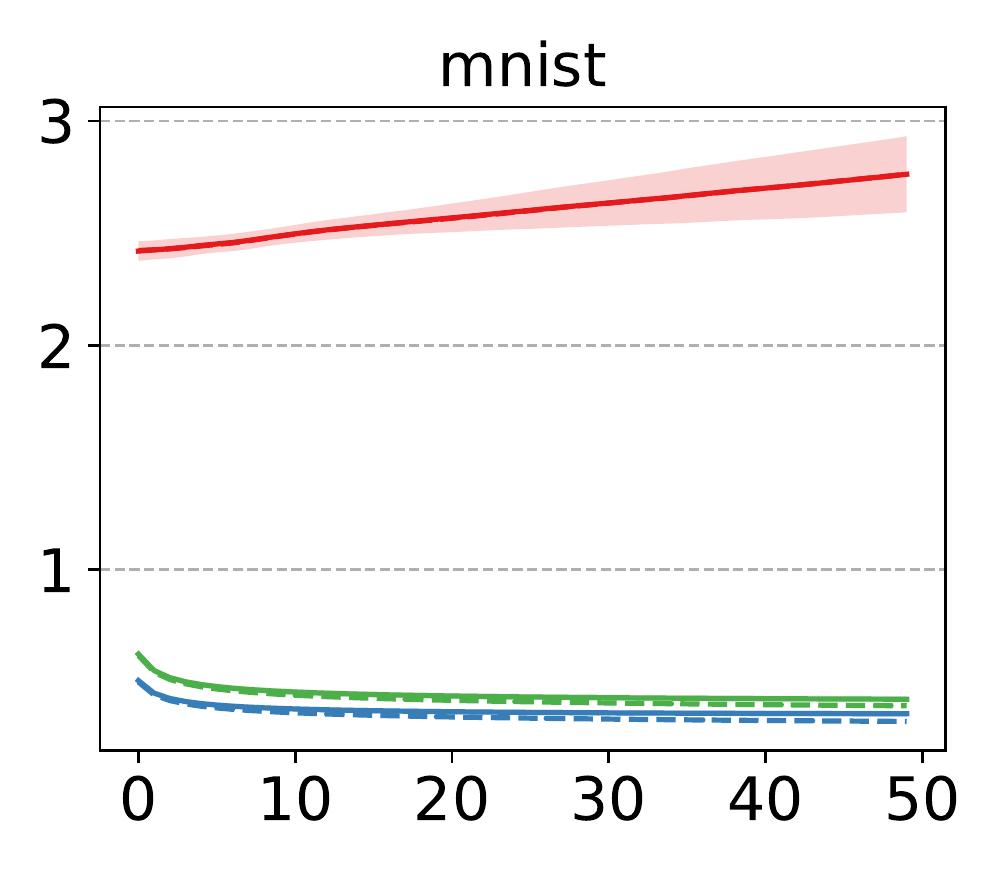}\includegraphics[width=0.25\textwidth]{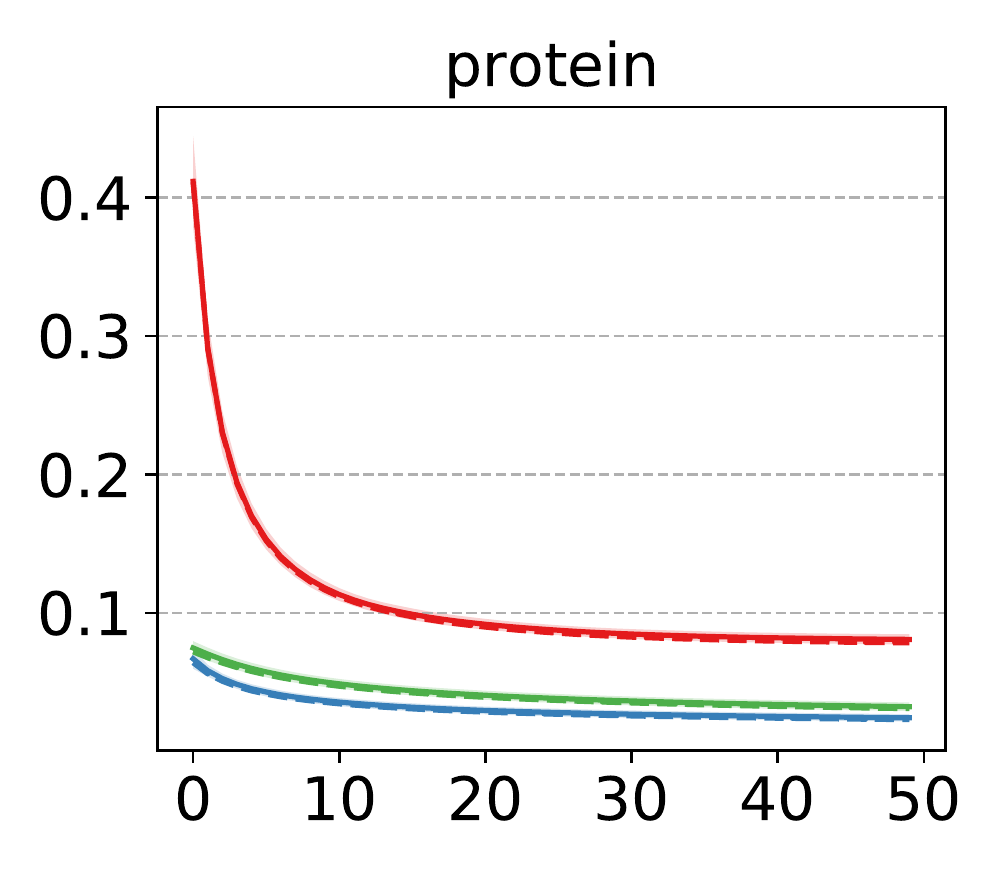}\\
\includegraphics[width=0.25\textwidth]{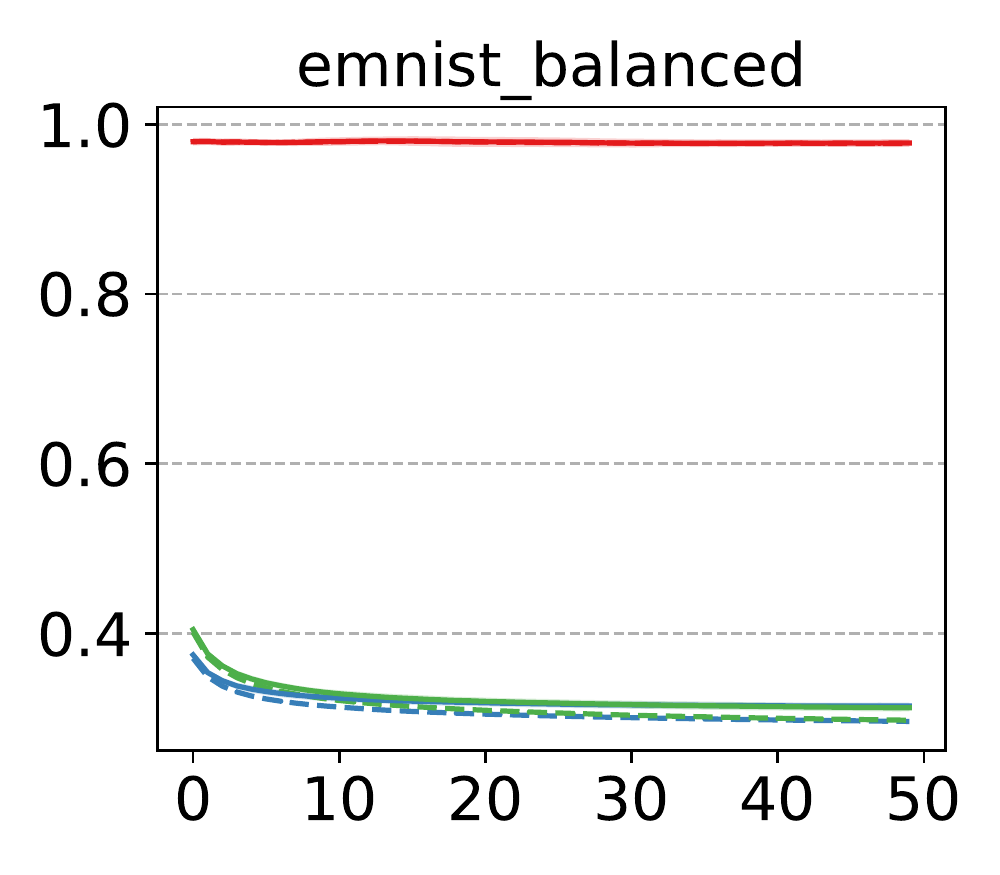}\includegraphics[width=0.25\textwidth]{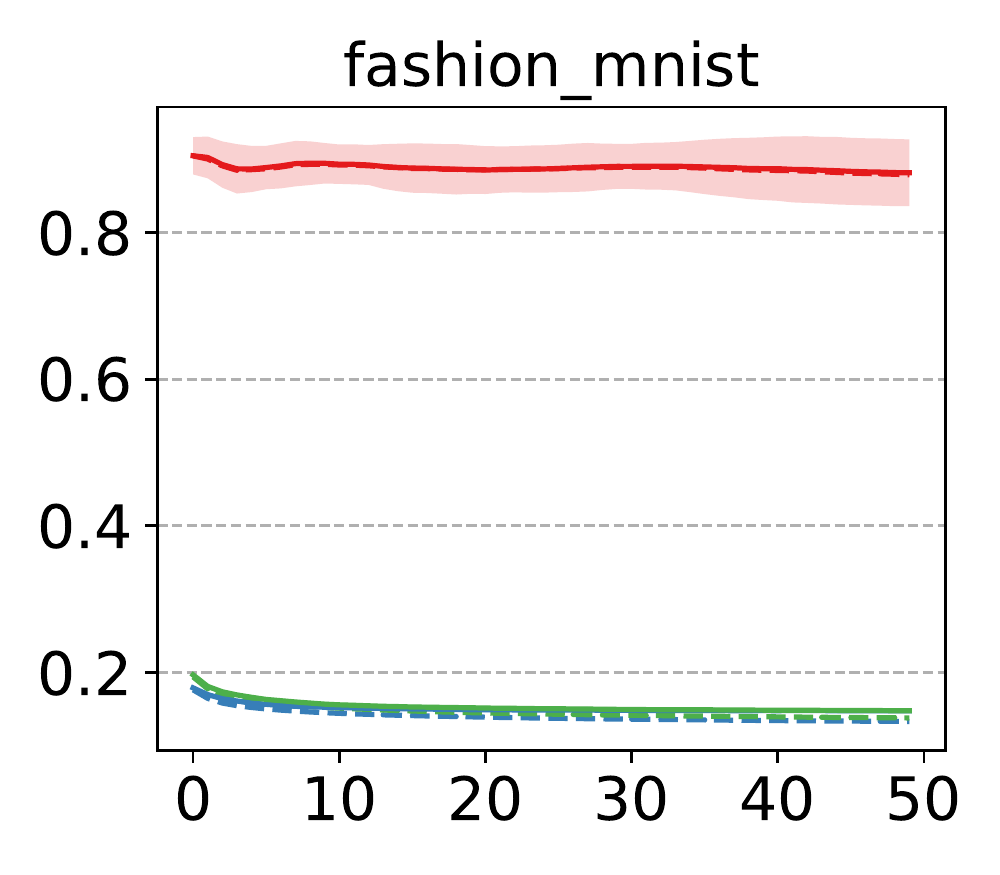}\includegraphics[width=0.25\textwidth]{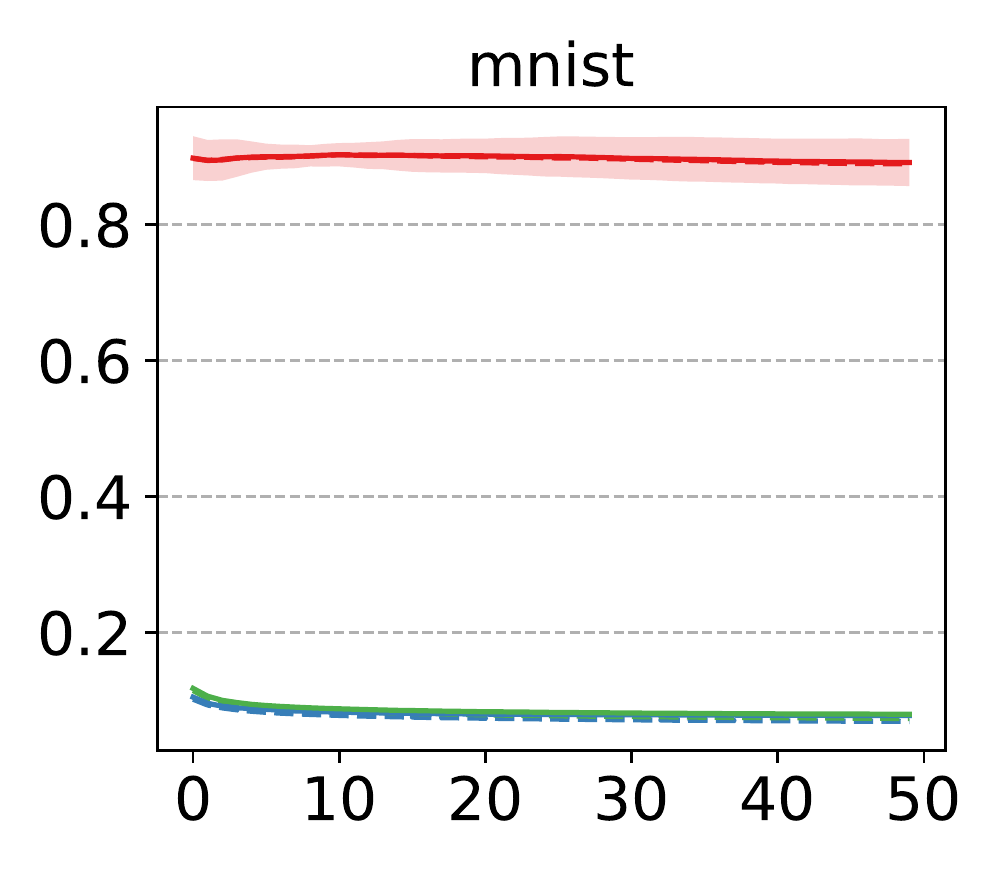}\includegraphics[width=0.25\textwidth]{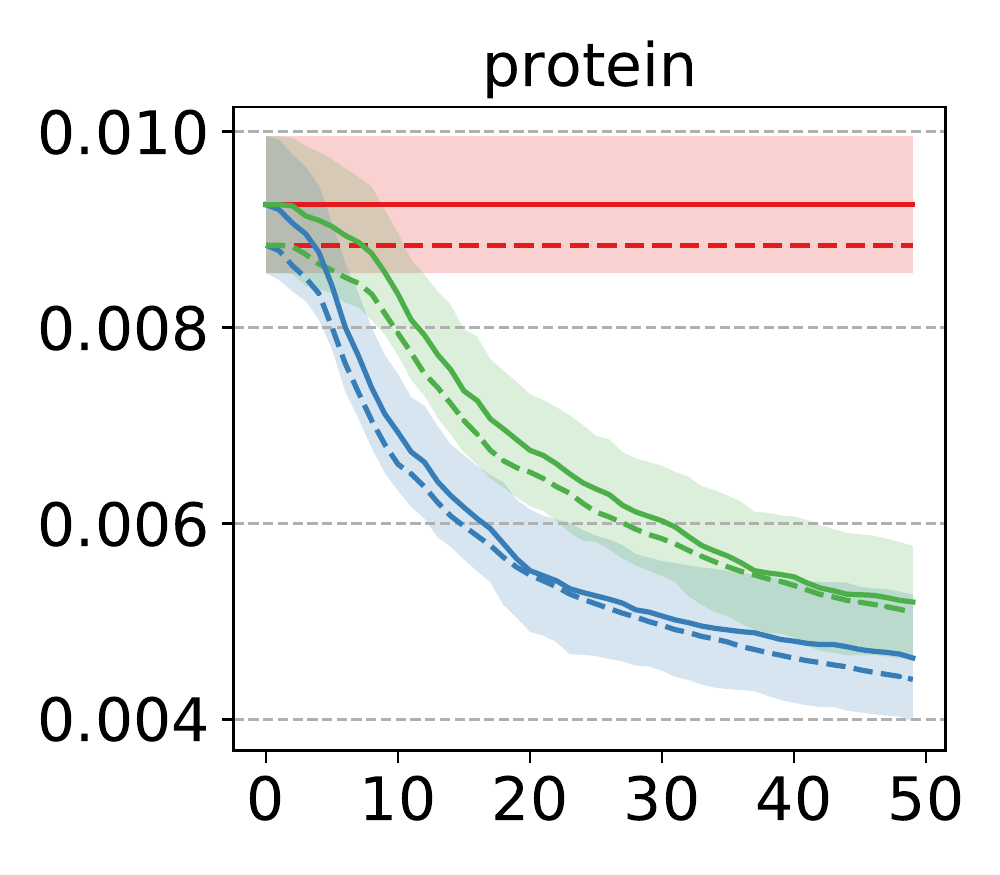}
\caption{Completely analogous to Figure \ref{fig:empirical}, for four additional datasets.}
\label{fig:empirical2}
\end{figure}

\section{Concluding remarks}

In this paper, we have studied a derivative-free learning procedure (Algorithm \ref{algo:smd}) with excess spectral risk guarantees, under losses that may be unbounded and heavy-tailed (Theorems \ref{thm:smd_srisk} and \ref{thm:highprob}), and provided a fast implementation which on numerous real-world classification tasks has been shown to be efficient without any hyperparameter tuning. Given the existing work on spectral risk estimation \citep{pandey2019a} and ERM for spectral risks \citep{khim2020a}, our results contribute to the literature by providing a transparent algorithmic solution for spectral risk-based learning, which is easy to implement and comes with lucid formal guarantees, plus a modified procedure that scales better to larger tasks.

Moving forward, the approach via Lemma \ref{lem:ssrisk_grad_link} relies crucially on Stokes' theorem on $\RR^{d}$, and lacks an analogue on richer spaces. Function space representations are useful in many learning methods \citep{dai2014a,nitanda2018a}, and extending Theorem \ref{thm:smd_srisk} to general Hilbert spaces is a point of interest. How should the noise be generated? How should derivatives be defined? While a direct analogue using differential theory (e.g., the Fr\'{e}chet differential \citep{penot2012CWOD}) appears difficult, an appeal to basic results in Malliavin calculus \citep{decreusefond2019} may open the door to a major generalization of the initial results established here.

\appendix

\section{Detailed proofs}

\subsection{Proofs from section \ref{sec:theory_expectation}}

\begin{proof}[Proof of Theorem \ref{thm:smd_srisk}]
At a high level, we first control $\ssrisk(\overbar{w}_T)$, and then using the fact that the functions $\ssrisk$ and $\srisk$ are close to each other on $\WW$, we can cast performance in terms of $\srisk$. This closeness depends on the perturbation factor $\pert$; smaller is closer. On the other hand, the smoothness coefficient of $\ssrisk$ grows as $\pert$ gets small, leading to a natural tradeoff. We begin the proof by showing this smoothness, which enables us to control $\ssrisk(\overbar{w}_T)$ in a straightforward manner, using an argument that relies upon well-known properties of mirror descent procedures.

\paragraph{Step 1: smoothness property of the smoothed spectral risk}

Recalling the expression (\ref{eqn:srisk_cvar}) for $\srisk$, if the map $w \mapsto \loss(w;Z)$ is convex and continuous on $\CC$, then so is $w \mapsto \crisk_{\beta}(w)$.\footnote{See for example \citet[Prop.~3.1, Lem.~3.1]{ruszczynski2006a}} From (\ref{eqn:srisk_cvar}), this immediately implies that $w \mapsto \srisk(w)$ is convex and continuous on $\CC$, and thus that there exists a constant $0 < \smooth_{\risk} \leq \sup_{v \in \CC} \srisk(v) < \infty$ such that
\begin{align}
\label{eqn:srisk_lipschitz}
|\srisk(v)-\srisk(v^{\prime})| \leq \smooth_{\risk}\|v - v^{\prime}\|
\end{align}
for all $v,v^{\prime} \in \CC$.\footnote{Since the closure of $\CC$ is compact, continuity implies that $\srisk$ is bounded above on $\CC$. The Lipschitz property follows from standard results, such as \citet[Prop.~3.8]{penot2012CWOD}.} Now turning our attention to the smoothed spectral risk $\ssrisk$, taking $U \sim \ndist_1$ and any $w,w^{\prime} \in \WW$, we write the resulting noisy parameters as $W \defeq w + \pert U$ and $W^{\prime} \defeq w^{\prime} + \pert U$. Using the key equality (\ref{eqn:ssrisk_grad_link}) along with (\ref{eqn:srisk_lipschitz}) just given, and the fact that $\|U\|=1$ almost surely $[\ndist_1]$, we have
\begin{align}
\nonumber
\| \nabla\ssrisk(w) - \nabla\ssrisk(w^{\prime}) \| & = \frac{d}{\pert} \left\| \exx_{\ndist_1}U\left[ \srisk(W) - \srisk(W^{\prime}) \right] \right\|\\
\nonumber
& \leq \frac{d}{\pert} \left| \srisk(W) - \srisk(W^{\prime}) \right|\\
\label{eqn:ssrisk_smooth}
& \leq \frac{d\smooth_{\risk}}{\pert} \|w - w^{\prime}\|.
\end{align}
As such, we can conclude that the smoothed spectral risk $\ssrisk$ is indeed $(d\smooth_{\risk}/\pert)$-smooth on $\WW$.

\paragraph{Step 2: idealized stochastic gradient}

As an idealized counterpart to $\widehat{G}_t$, we introduce $G_t \defeq (d/\pert)\sloss\left(w_t + \pert U_t;Z_t\right) U_t$. This is an ideal quantity in the sense that it is the stochastic gradient that would be obtained if the true distribution function $\dfun_t \defeq \dfun_{w_t}$ was known. Denote the ancillary datasets used in Algorithm \ref{algo:smd} by $\Z_{t}^{\prime} \defeq \{Z_{t,1}^{\prime},\ldots,Z_{t,M}^{\prime}\}$, for each step $t$, where $M$ is the specified size. Denote sub-sequences as $U_{[t]} \defeq (U_1,\ldots,U_t)$ for all $t>0$ (analogously for $Z_{[t]}$ and $\Z_{[t]}^{\prime}$), and write $\exx_{[t]}$ to denote taking expectation jointly over $(U_{[t]},Z_{[t]},\Z_{[t]}^{\prime})$. With this notation in place, note that taking expectation over all random elements, we can readily observe
\begin{align}
\nonumber
\exx \left[ G_t \right] & = \exx_{[t]} G_t\\
\nonumber
& = \exx_{[t-1]} \exx_{\ndist_1,\ddist} \left[ G_t \cond U_{[t-1]}, Z_{[t-1]}, \Z_{[t-1]}^{\prime} \right]\\
\nonumber
& = \exx_{[t-1]} \exx_{\ndist_1} \exx_{\ddist}\left[ G_t \cond U_{[t]}, Z_{[t-1]}, \Z_{[t-1]}^{\prime} \right]\\
\nonumber
& = \exx_{[t-1]} \exx_{\ndist_1} \left(\frac{d}{\pert}\right)\srisk(w_t + \pert U_t)U_t\\
\nonumber
& = \exx_{[t-1]} \nabla\ssrisk(w_t)\\
\label{eqn:grad_unbiased}
& = \exx \left[ \nabla\ssrisk(w_t) \right].
\end{align}
The first and last equalities hold because $G_t$ and $w_t$ are independent of all random quantities with index $t+1$ or larger. The second and third equalities use the law of total expectation.\footnote{See for example \citet[Thm.~5.3.3 and 5.5.4]{ash2000a}.} The rest just uses the definition of $\srisk$ and the unbiased property (\ref{eqn:ssrisk_grad_link}). This establishes that $G_t$ provides us with an unbiased estimate of the gradient of the smoothed spectral risk. Although the sequence $(G_t)$ is not directly observable, this unbiasedness will be technically useful.

\paragraph{Step 3: setup for mirror descent analysis}

As an intermediate step in the overall argument, we consider stochastic minimization of $\ssrisk$ using the procedure specified by (\ref{eqn:smd_actual}). Say we know that $\ssrisk$ is convex and $\smooth$-smooth on $\CC$.\footnote{In (\ref{eqn:ssrisk_smooth}) we have already been proved this holds with $\smooth = d\smooth_{\risk}/\pert$.} Taking advantage of smoothness and convexity, the following series of inequalities will make for a good starting point:
\begin{align}
\nonumber
\ssrisk&(w_{t+1}) - \ssrisk(w_t)\\
\nonumber
& \leq \langle \nabla\ssrisk(w_t), w_{t+1}-w_t \rangle + \frac{\smooth}{2} \|w_{t+1}-w_t\|^{2}\\
\nonumber
& = \langle \widehat{G}_t, w_{t+1}-w_t \rangle + \langle G_t-\widehat{G}_t, w_{t+1}-w_t \rangle + \langle \nabla\ssrisk(w_t)-G_t, w_{t+1}-w_t \rangle + \frac{\smooth}{2} \|w_{t+1}-w_t\|^{2}\\
\nonumber
& \leq \langle \widehat{G}_t, w_{t+1}-w_t \rangle + \left( \|G_t-\widehat{G}_t\| + \|\nabla\ssrisk(w_t)-G_t\| \right) \| w_{t+1}-w_t \| + \frac{\smooth}{2} \|w_{t+1}-w_t\|^{2}\\
\nonumber
& \leq \langle \widehat{G}_t, w_{t+1}-w_t \rangle + \frac{c}{2}\left( \|G_t-\widehat{G}_t\| + \|\nabla\ssrisk(w_t)-G_t\| \right)^{2} + \left( \frac{\smooth}{2} + \frac{1}{2c} \right) \|w_{t+1}-w_t\|^{2}\\
\label{eqn:min_ssrisk_init}
& \leq \langle \widehat{G}_t, w_{t+1}-w_t \rangle + c\left( \|G_t-\widehat{G}_t\|^{2} + \|\nabla\ssrisk(w_t)-G_t\|^{2} \right) + \left( \smooth + \frac{1}{c} \right) \frac{\breg(w_{t+1};w_t)}{\strong}.
\end{align}
The first inequality uses a basic property of functions with Lipschitz-continuous gradients.\footnote{See for example \citet[Thm.~2.1.5]{nesterov2004ConvOpt}.}  The second inequality is just Cauchy-Schwarz. The third inequality uses the elementary fact $2ab \leq ca^{2} + b^{2}/c$ for any $c>0$. The final inequality makes use of the fact that $(a+b)^{2} \leq 2(a^{2}+b^{2})$, and the fact that $\strong$-strong convexity of $\Phi$ implies $\breg(u;v) \geq (\strong/2)\|u-v\|^{2}$.

\paragraph{Step 4: bounding intermediate terms}

Taking the first term in (\ref{eqn:min_ssrisk_init}), fixing any $\widetilde{w}^{\ast} \in \RR^{d}$ we trivially have
\begin{align}
\label{eqn:min_ssrisk_intermediate_1}
\langle \widehat{G}_t, w_{t+1}-w_t \rangle = \langle \widehat{G}_t, w_{t+1}-\widetilde{w}^{\ast} \rangle + \langle G_t, \widetilde{w}^{\ast}-w_t \rangle + \langle \widehat{G}_t-G_t, \widetilde{w}^{\ast}-w_t \rangle.
\end{align}
Taking the right-hand side one term at a time, the first term is bounded by
\begin{align}
\label{eqn:min_ssrisk_intermediate_2}
\langle \widehat{G}_t, w_{t+1}-\widetilde{w}^{\ast} \rangle \leq \frac{A_t}{\alpha_t} \defeq \frac{\breg(\widetilde{w}^{\ast};w_t) - \breg(w_{t+1};w_t) - \breg(\widetilde{w}^{\ast};w_{t+1})}{\alpha_t},
\end{align}
a fact which holds from standard mirror descent analysis.\footnote{See \citet[Ch.~4, 6]{bubeck2015a} or \citet[Ch.~6]{orabona2020a} for a highly readable background.} Next, taking expectation over the second term, using (\ref{eqn:grad_unbiased}) and the convexity of $\ssrisk$, we have
\begin{align}
\label{eqn:min_ssrisk_intermediate_3}
\exx\left[ \langle G_t, \widetilde{w}^{\ast}-w_t \rangle + \ssrisk(w_t) \right] = \exx\left[ \langle \nabla\ssrisk(w_t), \widetilde{w}^{\ast}-w_t \rangle + \ssrisk(w_t) \right] \leq \ssrisk(\widetilde{w}^{\ast}).
\end{align}
Finally, to deal with the remaining gradient difference term, note that
\begin{align*}
\|\widehat{G}_t-G_t\| & \leq \frac{d}{\pert} \|U_t\| \loss_t | \spec(\dfun_t(\loss_t)) - \spec(\dfunhat_t(\loss_t)) |\\
& \leq \left(\frac{d\smooth_{\spec}\loss_t}{\pert}\right) \sup_{u \in \RR} | \dfun_t(u) - \dfunhat_t(u) |.
\end{align*}
Write $\exx_t^{\prime}$ to denote taking expectation with respect to $\Z_t^{\prime}$, and for readability, write the distribution function estimation error as $\|\dfun_t - \dfunhat_t\| \defeq \sup_{u \in \RR}|\dfun_t(u) - \dfunhat_t(u)|$. If we take the expectation of the inequality just derived and use Cauchy-Schwarz, we obtain
\begin{align*}
\exx \left[ \langle \widehat{G}_t-G_t, \widetilde{w}^{\ast}-w_t \rangle \right] & \leq \exx \left[ \left(\frac{d\smooth_{\spec}\loss_t}{\pert}\right) \|\widetilde{w}^{\ast}-w_t\| \|\dfun_t - \dfunhat_t\| \right]\\
& = \exx_{[t]} \left[ \left(\frac{d\smooth_{\spec}\loss_t}{\pert}\right) \|\widetilde{w}^{\ast}-w_t\| \|\dfun_t - \dfunhat_t\| \right]\\
& = \left(\frac{d\smooth_{\spec}}{\pert}\right) \exx \left[ \exx_t^{\prime} \left[ \left. \loss_t \|\widetilde{w}^{\ast}-w_t\| \|\dfun_t - \dfunhat_t\| \,\right|\, U_{[t-1]}, Z_{[t]}, \Z_{[t-1]}^{\prime} \right] \right]\\
& = \left(\frac{d\smooth_{\spec}}{\pert}\right) \exx \left[ \loss_t \|\widetilde{w}^{\ast}-w_t\| \exx_t^{\prime} \left[ \left. \|\dfun_t - \dfunhat_t\| \,\right|\, U_{[t-1]}, Z_{[t-1]}, \Z_{[t-1]}^{\prime} \right] \right].
\end{align*}
The above equalities follow from applying the law of total expectation and noting that conditioned on $U_{[t-1]}, Z_{[t-1]}, \Z_{[t-1]}^{\prime}$, $w_t$ is no longer random, and conditioned on $U_{[t-1]}, Z_{[t]}, \Z_{[t-1]}^{\prime}$, $\loss_t$ is no longer random. To clean up this upper bound, first note that
\begin{align*}
\exx_t^{\prime} \left[ \left. \|\dfun_t - \dfunhat_t\| \,\right|\, U_{[t-1]}, Z_{[t-1]}, \Z_{[t-1]}^{\prime} \right] & = \int_{0}^{\infty} \prr\left\{ \left. \|\dfun_t - \dfunhat_t\| > \varepsilon \,\right|\, U_{[t-1]}, Z_{[t-1]}, \Z_{[t-1]}^{\prime}  \right\} \, \dif\varepsilon\\
& \leq 2 \int_{0}^{\infty} \exp\left(-2M\varepsilon^{2}\right) \, \dif\varepsilon\\
& = \sqrt{\frac{\pi}{2M}}.
\end{align*}
The first equality is a basic probability result.\footnote{See \citet{lo2018a} for a lucid elementary background on this fact.} The inequality is just an application of the refined DKW inequality.\footnote{See for example \citet[Thm.~11.6]{kosorok2008EPSP}.} In a similar fashion, using $\diameter$ to bound the diameter of the hypothesis class $\WW$, we have that
\begin{align}
\nonumber
\exx \left[ \langle \widehat{G}_t-G_t, \widetilde{w}^{\ast}-w_t \rangle \right] & \leq \left(\frac{d\smooth_{\spec}}{\pert}\right) \sqrt{\frac{\pi}{2n}} \exx \loss_t \|\widetilde{w}^{\ast}-w_t\|\\
\nonumber
& \leq \left(\frac{d\smooth_{\spec}\diameter}{\pert}\right) \sqrt{\frac{\pi}{2n}} \exx \left[ \exx_\ddist\left[ \loss(w_t;Z) \cond U_{[t-1]}, Z_{[t-1]}, \Z_{[t-1]}^{\prime} \right] \right]\\
\nonumber
& = \left(\frac{d\smooth_{\spec}\diameter}{\pert}\right) \sqrt{\frac{\pi}{2n}} \exx \left[ \risk(w_t) \right]\\
\label{eqn:min_ssrisk_intermediate_4}
& \leq \left(\frac{d\smooth_{\spec}\smooth_{\risk}\diameter}{\pert}\right) \sqrt{\frac{\pi}{2n}}.
\end{align}
The final inequality uses the definition of $\smooth_{\risk}$ and the fact that $\WW \subset \CC$. This covers the first term in (\ref{eqn:min_ssrisk_init}).

\paragraph{Step 5: more intermediate terms}

For the second term in (\ref{eqn:min_ssrisk_init}), we need control of $\exx\|G_t - \widehat{G}_t\|^{2}$ and $\exx\|\nabla\ssrisk(w_t) - G_t\|^{2}$. As a simple bound on the first of these, noting that $\|\dfun_t-\dfunhat_t\| \leq 1$, we have
\begin{align*}
\exx\|G_t - \widehat{G}_t\|^{2} & \leq \left(\frac{d\smooth_{\spec}}{\pert}\right)^{2} \sup_{v \in \CC} \exx_{\ddist}|\loss(v;Z)|^{2}.
\end{align*}
For the remaining term, we have
\begin{align*}
\exx_{\ndist_1,\ddist}\|\nabla\ssrisk(w_t) - G_t\|^{2} & = \left(\frac{d}{\pert}\right)^{2} \exx_{\ndist_1,\ddist}\left[ \left\| \exx_{\ndist_1,\ddist}\left[\sloss(w_t+\pert U;Z)U\right] - \sloss(w_t+\pert U;Z)U \right\|^{2} \right]\\
& \leq \left(\frac{d}{\pert}\right)^{2} \sup_{v \in \CC} \exx_{\ndist_1,\ddist}\left[ \left\| \exx_{\ndist_1,\ddist}\left[\sloss(v;Z)U\right] - \sloss(v;Z)U \right\|^{2} \right].
\end{align*}
The preceding inequality holds because $0 < \pert < 1$ implies $w_t + \pert U \in \CC$ almost surely $[\ndist_1]$. Taking expectation over all elements and using the definitions of $s_1$ and $s_2$, we have
\begin{align}
\label{eqn:min_ssrisk_intermediate_5}
\exx \left[ c\left( \|G_t - \widehat{G}_t\|^{2} + \|\nabla\ssrisk(w_t) - G_t\|^{2} \right) \right] \leq c \left(\frac{d}{\pert}\right)^{2} \left( \left(\smooth_{\spec}s_2\right)^{2} + s_1^{2} \right).
\end{align}

\paragraph{Step 6: cleanup to bound smoothed spectral risk}

To start the cleanup process, taking inequalities (\ref{eqn:min_ssrisk_intermediate_2})--(\ref{eqn:min_ssrisk_intermediate_5}) back to (\ref{eqn:min_ssrisk_init}) and taking expectation, we can immediately deduce
\begin{align*}
\exx\left[ \ssrisk(w_{t+1}) - \ssrisk(\widetilde{w}^{\ast}) \right] & \leq \exx\left[ \frac{A_t}{\alpha_t} + \left( \smooth + \frac{1}{c} \right) \frac{\breg(w_{t+1};w_t)}{\strong} \right]\\
& \qquad\qquad + \left(\frac{d\smooth_{\spec}\smooth_{\risk}\diameter}{\pert}\right) \sqrt{\frac{\pi}{2M}} + c \left(\frac{d}{\pert}\right)^{2}\left( s_1^{2} + (\smooth_{\spec}s_2)^{2} \right).
\end{align*}
For the first term in the preceding inequality, since $A_t$ is composed of a difference of Bregman divergences, note that
\begin{align}
\nonumber
\frac{A_t}{\alpha_t} + \left(\smooth + \frac{1}{c}\right)&\frac{\breg(w_{t+1};w_t)}{\strong} \\
\nonumber
& = \frac{\breg(\widetilde{w}^{\ast};w_t)-\breg(\widetilde{w}^{\ast};w_{t+1})}{\alpha_t} + \breg(w_{t+1};w_t)\left(\frac{1}{\strong}\left(\frac{1}{c}+\smooth\right)-\frac{1}{\alpha_t}\right)\\
\label{eqn:min_ssrisk_cleanup_1}
& = \frac{\breg(\widetilde{w}^{\ast};w_t)-\breg(\widetilde{w}^{\ast};w_{t+1})}{\alpha(c)}.
\end{align}
The last equality holds via the setting of $\alpha_t = \alpha(c) \defeq \strong(\smooth+1/c)^{-1}$ for all $t$, causing the extra term to vanish. Next, leveraging Jensen's inequality and cancelling terms via the telescoping sum, we have
\begin{align*}
\exx&\left[ \ssrisk\left( \frac{1}{T} \sum_{t=1}^{T}w_{t} \right) - \ssrisk(\widetilde{w}^{\ast}) \right]\\
& \leq \exx\left[ \frac{1}{T}\sum_{t=1}^{T} \left( \ssrisk(w_t) - \ssrisk(\widetilde{w}^{\ast}) \right) \right]\\
& \leq \frac{\breg(\widetilde{w}^{\ast};w_1)-\breg(\widetilde{w}^{\ast};w_{T+1})}{T\alpha(c)} + \left(\frac{d\smooth_{\spec}\smooth_{\risk}\diameter}{\pert}\right) \sqrt{\frac{\pi}{2M}} + c\left( s_1^{2} + (\smooth_{\spec}s_2)^{2} \right)\\
& \leq \frac{\diameter_{\Phi}}{T\strong}\left(\smooth + \frac{1}{c}\right) + \left(\frac{d\smooth_{\spec}\smooth_{\risk}\diameter}{\pert}\right) \sqrt{\frac{\pi}{2M}} + c \left(\frac{d}{\pert}\right)^{2}\left( s_1^{2} + (\smooth_{\spec}s_2)^{2} \right).
\end{align*}
Minimizing the preceding upper bound with respect to $c>0$, one sets
\begin{align*}
c = \left(\frac{\pert}{d}\right)\sqrt{\frac{2\diameter_{\Phi} \strong}{T(s_1^{2}+(\smooth_{\spec} s_2)^{2})}}
\end{align*}
and obtains the bound
\begin{align}
\label{eqn:min_ssrisk_final}
\exx\left[ \ssrisk\left(\frac{1}{T} \sum_{t=1}^{T}w_{t}\right) - \ssrisk(\widetilde{w}^{\ast}) \right] \leq \left(\frac{d}{\pert}\right) \sqrt{\frac{2\diameter_{\Phi}(s_1^{2}+(\smooth_{\spec} s_2)^{2})}{T\strong}} + \frac{\smooth\diameter_{\Phi}}{T\strong} + \left(\frac{d\smooth_{\spec}\smooth_{\risk}\diameter}{\pert}\right) \sqrt{\frac{\pi}{2M}}.
\end{align}
Again, we remark that this holds for any fixed choice of $\widetilde{w}^{\ast}$.

\paragraph{Step 7: guarantees in terms of spectral risk}

Using (\ref{eqn:min_ssrisk_final}) we have a bound in expectation on the smoothed spectral risk $\ssrisk$ incurred by the (averaged) learning algorithm (\ref{eqn:smd_actual}), so it remains for us to relate this to the original objective of interest, namely the spectral risk $\srisk$. Denote a minimizer of this objective by $w^{\ast} \in \argmin_{w \in \WW} \srisk(w)$, and now let us fix $\widetilde{w}^{\ast}$ that appears in (\ref{eqn:min_ssrisk_final}) to be optimal in terms of $\ssrisk$, that is, let $\widetilde{w}^{\ast} \in \argmin_{w \in \WW} \ssrisk$ hold. Using this optimality and continuity properties of convex $\srisk$, we see that
\begin{align}
\nonumber
\srisk(\overbar{w}_T)-\srisk(w^{\ast}) & = \left[ \srisk(\overbar{w}_T)-\ssrisk(\overbar{w}_T) \right] + \left[\ssrisk(\overbar{w}_T)-\ssrisk(\widetilde{w}^{\ast})\right] + \left[ \ssrisk(\widetilde{w}^{\ast})-\srisk(w^{\ast}) \right]\\
\nonumber
& \leq 2 \sup_{w \in \WW} \left| \srisk(w) - \ssrisk(w) \right| + \ssrisk(\overbar{w}_T)-\ssrisk(\widetilde{w}^{\ast})\\
\nonumber
& = 2 \sup_{w \in \WW} \left| \exx_{\ndist}(\srisk(w)-\srisk(w + \pert U)) \right| + \ssrisk(\overbar{w}_T)-\ssrisk(\widetilde{w}^{\ast})\\
\label{eqn:min_srisk}
& \leq 2\smooth_{\risk}\pert + \ssrisk(\overbar{w}_T)-\ssrisk(\widetilde{w}^{\ast}).
\end{align}
The first inequality follows due to the optimality of $\widetilde{w}^{\ast}$, which implies $\ssrisk(\widetilde{w}^{\ast}) \leq \ssrisk(w^{\ast})$. The second equality follows from the definition of $\ssrisk$. The last inequality follows from (\ref{eqn:srisk_lipschitz}) and the fact that $\exx_{\ndist}\|U\| \leq 1$. Taking expectation of (\ref{eqn:min_srisk}), a direct application of the bound (\ref{eqn:min_ssrisk_final}) with $\lambda$ set according to (\ref{eqn:ssrisk_smooth}) yields the desired result.
\end{proof}

\subsection{Proofs from section \ref{sec:theory_highprob}}

\begin{proof}[Proof of Theorem \ref{thm:highprob}]
We start by proving inequality (\ref{eqn:highprob_helper}), namely the key validation error bound. After bounding $|\widehat{\risk}_{\spec}^{(j)} - \risk_{\spec}^{(j)}|$ by the two difference terms, the second inequality follows immediately from the definition of the spectral risk and the intermediate quantity $\overbar{\risk}_{\spec}^{(j)}$, using the $\smooth_{\spec}$-Lipschitz property of $\spec$ to get the error in terms of the error between distribution functions.

The next step (leading to (\ref{eqn:highprob_helper})) is comprised of a few parts. First, using H\"{o}lder's inequality, for any $w \in \CC$ we have $\exx_{\ddist}|\loss(w;Z)| \leq \sqrt{\exx_{\ddist}|\loss(w;Z)|^{2}} \leq s_2$, by definition of $s_2$. Next, for any fixed $w$, the DKW inequality \citep[Thm.~11.6]{kosorok2008EPSP} implies
\begin{align*}
\prr\left\{ \sup_{u}|\dfunhat_{w}(u) - \dfun_{w}(u)| > \varepsilon \right\} \leq 2\exp(-2\varepsilon^{2}\lfloor n/(k+1)\rfloor).
\end{align*}
Thus, conditioned on $\overbar{w}^{(j)}$, the bound on the second term in (\ref{eqn:highprob_helper}) holds with probability no less than $1-\delta/2$, over the random draw of the points used to compute the estimate $\dfunhat_{w}$. This is the first ``good event'' of interest.

The second good event is with respect to the remaining data $\{Z_{i}^{\prime\prime}\}$ used to compute the spectral risk estimates. Let us denote the variance of the weighted loss by
\begin{align*}
v_{\spec}^{(j)} \defeq \vaa_{\ddist}\left[ \loss(\overbar{w}^{(j)};Z)\spec(\dfunhat_{\overbar{w}^{(j)}}(\loss(\overbar{w}^{(j)};Z))) \right].
\end{align*}
Conditioning on $\dfunhat_{w}$ and $\overbar{w}^{(j)}$ for the moment, standard concentration inequalities for M-estimators tell us that
\begin{align}\label{eqn:highprob_1}
| \widehat{\risk}_{\spec}^{(j)} - \overbar{\risk}_{\spec}^{(j)} | \leq 2\sqrt{\frac{2v_{\spec}^{(j)}(1+\log(2\delta^{-1}))}{\lfloor n/(k+1)\rfloor}}
\end{align}
holds with probability no less than $1-\delta/2$; see for example \citet{catoni2012a} or \citet{devroye2016a} for typical examples of $\rho$ and $b$ settings. To get a bound free of the elements being conditioned upon, note that the variance of the weighted loss can be bounded as
\begin{align*}
\vaa_{\ddist}\loss(w;Z)\spec(\dfunhat_{w}(\loss(w;Z))) \leq \overbar{\spec}^{2}\exx_{\ddist}|\loss(w;Z)|^{2} \leq \overbar{\spec}^{2}s_{2}^{2} < \infty.
\end{align*}
We can thus bound $v_{\spec}^{(j)} \leq \overbar{\spec}^{2}s_{2}^{2}$ in (\ref{eqn:highprob_1}), and this is our second good event of interest. Taking a union bound of these two ``good events'' (each with probability at least $1-\delta/2$), we obtain (\ref{eqn:highprob_helper}) with probability at least $1-\delta$, as desired.

With inequality (\ref{eqn:highprob_helper}) in hand for each of the sub-processes indexed by $j=1,\ldots,k$, we can combine this with the key learning guarantees in expectation provided by Theorem \ref{thm:smd_srisk}. In particular, we use the excess expected spectral risk bound $\varepsilon_{1}(\cdot)$ in (\ref{eqn:sample_complexity}), but this time passed a sample of size $n/(k+1)$, since that is all that each sub-process (each independent run of Algorithm \ref{algo:smd}) is allocated. The desired result then follows quite mechanically using a generic robust confidence boosting argument, namely plugging in (\ref{eqn:sample_complexity}) and (\ref{eqn:highprob_helper}) to \citep[Lem.~9]{holland2021c}, and setting $k$ as specified allows us to clean up the probabilities \citep[Thm.~4 proof]{holland2021c}, yielding the $1-3\delta$ probability good event given in our theorem statement.
\end{proof}

\section{Additional empirical details}

Our empirical tests have been implemented in Python (v.~3.8) with the following open-source software: matplotlib (v.~3.4.1), PyTables (v.~3.6.1), Jupyter notebook, NumPy (v.~1.20.0), and SciPy (v.~1.6.2, for special functions). See Table \ref{table:datasets} for URLs to online documentation for each of the datasets used in our experiments. As discussed in the main text, we use multi-class logistic regression, with one linear model for each class, so the number of parameters to be determined is the number of classes (e.g., $2$ for \texttt{adult}, $47$ for \texttt{emnist\_balanced}) multiplied by the number of input features (e.g., $105$ for \texttt{adult}, $784$ for \texttt{emnist\_balanced}). Categorical features are given a one-hot representation, and all input features are standardized to take values on the unit interval $[0,1]$.

\begin{table}[t!]
\begin{center}
\begin{tabular}{|l|l|}
\hline
Dataset & URL \\
\hline\hline
\texttt{adult} & \url{https://archive.ics.uci.edu/ml/datasets/Adult}\\
\hline
\texttt{cifar10} & \url{https://www.cs.toronto.edu/~kriz/cifar.html}\\
\hline
\texttt{cod\_rna} & \url{https://www.csie.ntu.edu.tw/~cjlin/libsvmtools/datasets/binary.html}\\
\hline
\texttt{covtype} & \url{https://archive.ics.uci.edu/ml/datasets/covertype}\\
\hline
\texttt{emnist\_balanced} & \url{https://www.nist.gov/itl/products-and-services/emnist-dataset}\\
\hline
\texttt{fashion\_mnist} & \url{https://github.com/zalandoresearch/fashion-mnist}\\
\hline
\texttt{mnist} & \url{http://yann.lecun.com/exdb/mnist/}\\
\hline
\texttt{protein} & \url{https://www.kdd.org/kdd-cup/view/kdd-cup-2004/Data}\\
\hline
\end{tabular}
\end{center}
\caption{Benchmark dataset summary.}
\label{table:datasets}
\end{table}

\bibliographystyle{../refs/apalike}
\bibliography{../refs/refs.bib}

\end{document}